\pdfoutput=1
\documentclass{article}

\usepackage[final,nonatbib]{neurips2022/neurips_2022}

\usepackage[utf8]{inputenc} 
\usepackage[T1]{fontenc}    
\usepackage{url}            
\usepackage{booktabs}       
\usepackage{amsfonts}       
\usepackage{nicefrac}       
\usepackage{microtype}      

\usepackage{amsmath}
\usepackage{amssymb}
\usepackage{mathtools}
\usepackage{amsthm}

\usepackage[usenames,dvipsnames,svgnames]{xcolor} 

\usepackage[utf8]{inputenc} 
\usepackage[T1]{fontenc}    
\usepackage{url}            
\usepackage{booktabs}       
\usepackage{amsfonts}       
\usepackage{nicefrac}       
\usepackage{microtype}      

\usepackage{enumitem}
\usepackage[numbers]{natbib}
\usepackage{color}
\usepackage{placeins}
\usepackage{amsmath}
\usepackage{amsthm}
\usepackage{amssymb}
\usepackage{caption}
\usepackage{tikz}
\usepackage{comment}
\usepackage{bm}
\usepackage{array}
\usepackage{times}
\usepackage{scalefnt}
\usepackage{mathtools}
\usepackage{footnote}
\usepackage{amsmath}
\usepackage{amssymb}
\usepackage{graphicx}

\usepackage{algorithm}
\usepackage[noend]{algpseudocode}

\usepackage{wrapfig}

\usepackage[colorlinks=true, linkcolor=red, urlcolor=NavyBlue, citecolor=NavyBlue]{hyperref}

\usepackage[textsize=tiny]{todonotes}

\theoremstyle{plain}
\newtheorem{theorem}{Theorem}[section]

\theoremstyle{definition}
\newtheorem{definition}[theorem]{Definition}

\theoremstyle{remark}

\newcommand{\Ac}{{\mathcal{A}}}

\newcommand{\Dc}{{\mathcal{D}}}

\newcommand{\Fc}{{\mathcal{F}}}

\newcommand{\Nc}{{\mathcal{N}}}

\newcommand{\Pc}{{\mathcal{P}}}

\newcommand{\Sc}{{\mathcal{S}}}

\newcommand{\Xc}{{\mathcal{X}}}
\newcommand{\Yc}{{\mathcal{Y}}}

\newcommand{\Eb}{{\mathbb{E}}}

\newcommand{\Nbb}{{\mathbb{N}}}

\newcommand{\Rbb}{{\mathbb{R}}}

\newcommand{\bfa}{{\mathbf{a}}}

\newtheorem{prop}{Proposition}

\DeclareMathOperator*{\argmax}{arg\,max}
\DeclareMathOperator*{\argmin}{arg\,min}

\def\expec#1#2{{\mathbb{E}}_{#1}[ #2 ]}
\def\expecf#1#2{{\mathbb{E}}_{#1}\left[ #2 \right]}

\def\truefunction{\mathbf{f}}

\def\eqref#1{Eq.~(\ref{#1})}

\newcommand{\tx}{{\mathfrak{x}}}
\newcommand{\ta}{{\mathfrak{a}}}

\title{Generalizing Bayesian Optimization with Decision-theoretic Entropies}

\author{
    Willie Neiswanger\thanks{The first two authors contributed equally to this work.}~~, Lantao Yu$^*$, Shengjia Zhao, Chenlin Meng, Stefano Ermon\\
    Computer Science Department, Stanford University\\
    Stanford, CA 94305\\
    \texttt{\{neiswanger,lantaoyu,sjzhao,chenlin,ermon\}@cs.stanford.edu}\\
}

\begin{document}

\maketitle

\vspace{-2mm}
\begin{abstract}
Bayesian optimization (BO) is a popular method for efficiently inferring optima of an expensive black-box function via a sequence of queries. Existing information-theoretic BO procedures aim to make queries that most reduce the uncertainty about optima, where the uncertainty is captured by Shannon entropy. However, an optimal measure of uncertainty would, ideally, factor in how we intend to use the inferred quantity in some downstream procedure. In this paper, we instead consider a generalization of Shannon entropy from work in statistical decision theory \citep{DeGroot1962-ob, rao1984convexity}, which contains a broad class of uncertainty measures parameterized by a problem-specific loss function corresponding to a downstream task. We first show that special cases of this entropy lead to popular acquisition functions used in BO procedures such as knowledge gradient, expected improvement, and entropy search. We then show how alternative choices for the loss yield a flexible family of acquisition functions that can be customized for use in novel optimization settings. Additionally, we develop gradient-based methods to efficiently optimize our proposed family of acquisition functions, and demonstrate strong empirical performance on a diverse set of sequential decision making tasks, including variants of top-$k$ optimization, multi-level set estimation, and sequence search.
\end{abstract}
\section{Introduction}
\label{sec:introduction}

\everypar{\looseness=-1}
Bayesian optimization (BO) is a popular method for efficient global
optimization of an expensive black-box function, which leverages a probabilistic
model to judiciously choose a sequence of function queries.
In BO, there are a few key paradigms that motivate existing methodologies.
One paradigm is
decision-theoretic BO, which includes methods such as
\textit{knowledge gradient} \citep{frazier2009knowledge} and
\textit{expected improvement} \citep{movckus1975bayesian, jones1998efficient}.
At each iteration of BO, these methods aim to make a query that maximally increases
the expected value, under the posterior, of a final estimate of the optima
(sometimes referred to as a \textit{terminal action}).
Another common paradigm is
based on maximal uncertainty reduction and includes
information-based BO methods such as the family of \textit{entropy search} methods
\cite{Hennig2012-zu, hernandez2014predictive, wang2017max, neiswanger2021bayesian}.
At each iteration of BO, these methods aim to make a query that most
reduces the uncertainty, under the posterior, about a quantity of interest
(such as the location of the optima).

In the uncertainty-reduction paradigm, the information-based methods have predominantly
used Shannon entropy as the measure of uncertainty.
While Shannon entropy is one measure of uncertainty that we could aim to reduce at each
iteration of BO,
it is not the only measure, and it is not necessarily the most ideal measure
for every optimization task.
For instance, an optimal uncertainty function would, ideally,
factor in how we intend to
\textit{use} the final uncertainty about an inferred quantity
in some downstream procedure.

In this paper, we develop a framework that aims to first unify and then extend these two paradigms.
Specifically, we adopt a generalized definition of entropy from past work in Bayesian
decision theory \citep{DeGroot1962-ob, rao1984convexity, Grunwald2004-ai}, which
proposes a family of \textit{decision-theoretic entropies} parameterized
by a problem-specific loss function and action set.
This family includes Shannon entropy as a special case.
Using this generalized entropy, we can view information-based BO methods as
instances of decision-theoretic BO, with
a terminal action chosen from a different type of action set.
Similarly, this framework
also includes as special cases the decision-theoretic methods such as
expected improvement and knowledge gradient, which yields an uncertainty-reduction view of these methods. 
Beyond this unified view, our framework can be easily adapted to novel problem settings by choosing an appropriate loss and action set tailored to a given downstream use case. This allows for handling new optimization scenarios that have not previously been studied and where no BO procedure currently exists.  

As an example, there are many real-world problems where we want to estimate a set of optimal points, rather than a single global optimum.
Use cases include when we wish to find a set of highest-value points 
subject to some constraints on the similarity between these points
(e.g. to produce a diverse set of candidates in drug or materials design
\cite{pyzer2018bayesian, terayama2021black, tran2021computational}),
or points which satisfy some sequential relation (e.g. to construct a library of molecules
that attains a sequence of desired measurements \citep{fong2021utilization}).
Further, we may wish to estimate other properties of a black-box function, such as certain curves, surfaces, or subsets of the domain \cite{zhong2020accelerated, konakovic2020diversity, singh2008nonparametric}.
Due to the vast number of possibilities, 
most custom problem settings have not been explicitly studied in the literature.
A key advantage of our framework is that it provides a way to approach 
these problems where no suitable methods have been developed.

Additionally, since we define this family of generalized entropies in a standardized way,
we can develop a common acquisition optimization procedure, which 
applies generically to many members of this family
(where each member is induced  by a specific loss function and action set).
In particular, we develop a fully differentiable acquisition optimization method
inspired by recent work on one-shot knowledge gradient procedures \citep{balandat2020botorch}.
This yields an effective and computationally efficient algorithm for many optimization
and sequential decision making tasks, as long as the problem-specific loss function is differentiable.
In summary, our main contributions are the following:
\begin{itemize}[parsep=0pt, topsep=0pt, itemsep=5pt, leftmargin=8mm]
    \item We propose an acquisition function based on a family of
    \textit{decision-theoretic entropies} parameterized by a
    loss function $\ell$ and action set $\Ac$.
    Under certain choices of $\ell$ and $\Ac$, we can view multiple BO 
    acquisition functions in a single decision-theoretic perspective, which sheds light
    on the settings for which each is best suited.
    
    \item By selecting a suitable $\ell$ and $\Ac$, we can produce a problem-specific
    acquisition function, which is tailored to a given downstream use case.
    This yields a customizable BO method
    that can be applied to
    new optimization problems and other sequential decision making tasks,
    where no applicable methods currently exist.
    \item We develop an acquisition optimization procedure
    that applies generically to many instance of our framework.
    This procedure is computationally efficient, using a gradient-based approach.
    \item We demonstrate that our method shows strong empirical performance on a diverse set of
    tasks including top-$k$ optimization with diversity, multi-level set estimation, and sequence
    search.
\end{itemize}
\section{Setup}
\label{sec:setup}

Let $\truefunction: \Xc \rightarrow \Yc \subset \Rbb$ denote an expensive black-box
function that maps from an input search space $\Xc$ to an output space
$\Yc$, and $\truefunction \in \Fc$.
We assume that we can evaluate $\truefunction$ at an input $x \in \Xc$,
and will observe a  noisy function value $y_x = \truefunction(x) + \epsilon$,
where $\epsilon \sim \Nc(0, \eta^2)$.

We also assume that our uncertainty about $\truefunction$ is captured by a 
probabilistic model with prior distribution $p(f)$, which reflects
our prior beliefs about $\truefunction$. Given a dataset of observed function
evaluations $\Dc_t = \{(x_i, y_{x_i})\}_{i=1}^{t-1}$, our model gives a posterior
distribution over $\Fc$, denoted by $p(f | \Dc_t)$.

Suppose that, after a given BO procedure is complete, we intend to choose
a terminal action $a$ from some set of actions $\Ac$, and then incur
a loss based on both this action $a$ and the function $\truefunction$.
We denote this loss as $\ell: \Fc \times \Ac \rightarrow \mathbb{R}$.
As one example, after the BO procedure, suppose we 
make a single guess for the function maximizer, and then incur a 
loss based on the value of the function at this guess. In this case,
the action set is $\Ac = \Xc$ and the loss is
$\ell(\truefunction, a) = -\truefunction(a)$.


\section{Decision-theoretic Entropy Search}
\label{sec:hentropysearch}

\everypar{\looseness=-1}

In this section, we first describe a family of \textit{decision-theoretic entropies} from
past work in Bayesian decision theory \citep{DeGroot1962-ob, rao1984convexity, Grunwald2004-ai},
which are parameterized by a problem-specific action set $\Ac$ and loss function $\ell$.
This family includes Shannon entropy as a special case.
We denote this family using the symbol $H_{\ell, \Ac}$, and refer to it as the
\textit{$H_{\ell, \Ac}$-entropy}.

\begin{definition} ($H_{\ell, \Ac}$-entropy of $f$).
\label{def:hentropy}
Given a prior distribution $p(f)$ on functions, and a dataset $\Dc$ of
observed function evaluations, the posterior $H_{\ell, \Ac}$-entropy with
loss $\ell$ and action set $\Ac$ is defined as
\begin{align}
    \label{eq:hentropy}
    H_{\ell, \Ac} \left[ f \mid \Dc \right] =
    \inf_{a \in \Ac} \expecf{p(f \mid \Dc)}{ \ell(f, a) }.
\end{align}
\end{definition}
Intuitively, after expending our budget of function queries, suppose
that we must make a terminal action $a^* \in \Ac$,
where this action incurs a loss $\ell(f, a^*)$ defined by the loss $\ell$
and function $f$. Given a posterior $p(f | \Dc)$ that
describes our belief about $f$ after observing $\Dc$,
we take the terminal action $a^*$ to be the \textit{Bayes action}, i.e.
the action that  minimizes the posterior expected loss,
$a^* = \arg\inf_{a \in \Ac} \expecf{p(f|\Dc)}{\ell(f, a)}$.
The $H_{\ell, \Ac}$-entropy can then be viewed as the posterior expected loss of the
Bayes action.
We next describe how this generalizes Shannon entropy, and why it is a reasonable definition
for an uncertainty measure. 

\vspace{-2mm}
\paragraph{Example: Shannon entropy}
Let $\Pc(\Fc)$ denote a set of probability distributions on a function space $\Fc$,
which we assume contains the posterior distribution $p(f \mid \Dc) \in \Pc(\Fc)$.
Suppose, for the $H_{\ell, \Ac}$-entropy, that
we let the action set $\Ac = \Pc(\Fc)$,
and loss function $\ell(f, a) = - \log a(f)$, for $a \in \Pc(\Fc)$.
Unlike the previous examples, note that the action set is now a set of distributions.

Then, the Bayes action will be $a^*$ $=$ $p(f \mid \Dc)$ (this can be shown by
writing out the definition of the Bayes action as a cross entropy,
see Appendix~\ref{sec:app-proof-es}), and thus
\begin{align}
H_{\ell, \Ac}[f \mid \Dc]
= \expecf{p(f \mid \Dc)}{- \log a^*(f) }
= H[f \mid \Dc],
\end{align}
where $H[f \mid \Dc] = - \int p(f \mid \Dc) \log p(f \mid \Dc)$ is the 
Shannon differential entropy.
Thus, the $H_{\ell, \Ac}$-entropy using the above $(\ell, \Ac)$ is equal to the
Shannon differential entropy.

Note that we have focused here on the Shannon entropy of the
posterior over functions $p(f \mid \Dc)$.
In Section~\ref{sec:existingacqfunctions} we show how this example can
be extended to the Shannon entropy
of the posterior over properties of $f$, such as
the location (or values) of optima,
which will provide a direct equivalence to entropy search methods in BO.

\vspace{-1mm}
\paragraph{Why is this a reasonable measure of uncertainty?}
The $H_{\ell, \Ac}$-entropy has been interpreted as a measurement of uncertainty in the
literature because it satisfies a few intuitive properties.
First, similar to Shannon differential entropy, the $H_{\ell, \Ac}$-entropy is a
\textit{concave uncertainty measure} \citep{DeGroot1962-ob, Grunwald2004-ai}.
Intuitively, if we have two distributions $p_1$ and $p_2$, and flip a coin to sample
from $p_1$ or $p_2$, then we should have less uncertainty if we were told the outcome
of the coin flip than if we weren't.
In other words, the average uncertainty of $p_1$ and $p_2$ (i.e. coin flip outcome \textit{known})
should be less than the uncertainty of $0.5p_1 + 0.5p_2$ (coin flip outcome \textit{unknown}).
Since $H_{\ell, \Ac}$ is concave, it has this property.
As a consequence---also similar to Shannon differential entropy---the
$H_{\ell, \Ac}$-entropy of the posterior is less than the
$H_{\ell, \Ac}$-entropy of the prior, in expectation.
Intuitively, whenever we make additional observations (i.e. gain more information),
the posterior entropy is expected to decrease.

\vspace{-1mm}
\paragraph{Acquisition function}
We propose a family of acquisition functions for BO based on the $H_{\ell, \Ac}$-entropy, 
which are similar in structure to information-theoretic acquisition
functions in the entropy search family.
Like these, our acquisition function selects the query
$x_t \in \Xc$ that maximally reduces the uncertainty, as characterized by
the $H_{\ell, \Ac}$-entropy, in expectation. We refer to this quantity as the
\textit{expected $H_{\ell, \Ac}$-information gain} (EHIG).

\vspace{1mm}
\begin{definition} (Expected $H_{\ell, \Ac}$-information gain).
\label{def:ehig}
Given a prior $p(f)$ on functions and a dataset of observed function evaluations 
$\Dc_t$, the expected $H_{\ell, \Ac}$-information gain (EHIG), with loss $\ell$ and action
set $\Ac$, is defined as
\begin{align}
    \label{eq:ehig}
    \text{EHIG}_t(x; \ell, \Ac) = 
    H_{\ell, \Ac} \left[ f \mid \Dc_t \right]
    - \expecf{p(y_x | \Dc_t)}{
        H_{\ell, \Ac} \left[f \mid \Dc_t \cup \{(x, y_x)\} \right]
    }.
\end{align}
\end{definition}
There are multiple benefits to developing this acquisition function.
Though similar in form to entropy search acquisition functions, 
the EHIG yields (based on the definition of $H_{\ell, \Ac}$) the one-step
Bayes optimal query
for the associated decision problem specified by the given loss $\ell$ and action set $\Ac$.
We prove in Section~\ref{sec:existingacqfunctions} that the EHIG
casts both uncertainty-reduction and decision-theoretic acquisition
functions under a common
umbrella, using different choices of $\ell$ and $\Ac$;
this standardization provides guidance on which acquisition function is
optimal for a given use case, based on details of the associated
terminal action.
More interestingly, in Section~\ref{sec:novelacqfunctions} we show how the
EHIG allows us to derive problem-specific acquisition functions tailored to novel
optimization and sequential decision making tasks.
And importantly, since we frame acquisition optimization of this family
in a common way---as a bilevel optimization problem over the sample space and action space---we
can develop a single acquisition optimization method that can generically apply to many
custom tasks (Section~\ref{sec:acqopt}).

In Algorithm~\ref{alg:hes}, we present \textsc{$H_{\ell, \Ac}$-Entropy Search},
our full Bayesian optimization procedure using the EHIG acquisition function.
This procedure takes as input a loss $\ell$, action set $\Ac$, and prior model $p(f)$.
At each iteration, the procedure optimizes $\text{EHIG}_t(x; \ell, \Ac)$ to select a 
design $x_t \in \Xc$ to query, and then evaluates the black-box function on this
design to observe an outcome $y_{x_t} \sim \truefunction(x_t) + \epsilon$.
In Section~\ref{sec:acqopt} we describe methods for
optimizing the EHIG acquisition function via gradient-based
procedures, which provide a computationally efficient algorithm
for many
$\Ac$ and $\ell$.

\begin{algorithm}
    \caption{\textsc{$H_{\ell, \Ac}$-Entropy Search}}
    \label{alg:hes}
    \textbf{Input:} initial dataset $\Dc_1$,
    prior $p(f)$, action set $\Ac$, loss $\ell$.
    \begin{algorithmic}[1]
      \For{$t = 1,\ldots,T$}
        \State $x_t \leftarrow \argmax_{x \in \Xc} \text{EHIG}_t(x; \ell, \Ac)$
            \Comment{Optimize the $\textrm{EHIG}$ acquisition function}
        \State $y_{x_t} \sim \truefunction(x_t) + \epsilon$
            \Comment{Evaluate the function $\truefunction$ at $x_t$}
        \State $\Dc_{t+1} \leftarrow \Dc_t \cup \{(x_t, y_{x_t})\}$
            \Comment{Update the dataset}
      \EndFor 
    \end{algorithmic}
    \textbf{Output:} distribution
    $p(f \mid \Dc_{T+1})$
\end{algorithm}
\section{A Unified View of Information-based and Decision-theoretic Acquisitions}
\label{sec:existingacqfunctions}

\everypar{\looseness=-1}

In this section, we aim to show how acquisition functions commonly used in 
BO are special cases of the proposed EHIG family,
for particular choices of $\ell$ and $\Ac$.
This will allow us to view each acquisition function (including information-based
ones) from the perspective of a common decision problem:
after the BO procedure is complete, we choose a terminal action from 
action set $\Ac$ and then incur a loss defined by $\ell$.
Each acquisition function can be viewed as reducing
the posterior uncertainty over $f$ in a way that yields a 
terminal action with lowest expected loss.

This unified view provides two main benefits.
First, it sheds light on the particular scenarios in which one of the existing 
acquisition functions is optimal over the others (which we focus on in this section).
Second, it shows how using the EHIG with other choices for $\ell$ and $\Ac$
provides new acquisition functions for a broader set of optimization scenarios and
related tasks (which is the focus of Section~\ref{sec:novelacqfunctions}).

\paragraph{Information-based acquisition functions}
We state the family of entropy search acquisitions function
in a general way that includes the entropy search (ES) \citep{Hennig2012-zu},
predictive entropy search (PES) \citep{hernandez2014predictive},
and max-value entropy search (MES) \citep{wang2017max} algorithms.
Let $\theta_f \in \Theta$ denote a property of $f$ we would like to infer.
For example, we could set $\theta_f = \argmax_{x \in \Xc} f(x) = x^* \in \Xc$,
i.e. the location of the global maximizer of $f$, or
$\theta_f = \max_{x \in \Xc} f(x) \in \Rbb$, i.e. the maximum value achieved by $f$ in $\Xc$.
This family of entropy search acquisition function can then be written as:
\begin{align}
    \text{ES}_t(x) = H \left[ \theta_f | \Dc_t \right]
    - \expecf{p(y_x | \Dc_t)}{
        H \left[ \theta_f | \Dc_t \cup \{(x, y_x)\} \right]
    }. \nonumber
\end{align}
We can view this acquisition function as a special case of the $\mathrm{EHIG}$
in the following way.
Suppose, after the BO procedure is complete, we choose a distribution $q$ from
a set of distributions $\Pc(\Theta)$ and then incur a loss 
equal to the negative log-likelihood of $q$ for the true value of
$\theta_f$.
In this case, we view the action set as $\Ac = \Pc(\Theta)$ and the loss
function as $\ell(f, a) = -\log a(\theta_f)$, where $a \in \Ac$.
To visualize this, in the case where $\theta_f = x^*$, see
Figure~\ref{fig:bayesactions} (left), which shows the terminal action 
(gold density function) and corresponding loss (horizontal dashed line).

Under this choice, the $H_{\ell, \Ac}$-entropy of $f$ will be equal to the
Shannon entropy of $\theta_f$, and thus the $\mathrm{EHIG}_t$ will be equal
to $\mathrm{ES}_t$.
We formalize this in the following proposition.

\vspace{1mm}
\begin{prop}
\label{prop:equiv_es} 
If we choose $\Ac = \Pc(\Theta)$ and $\ell(f, q) = - \log q( \theta_f )$, then the $\mathrm{EHIG}$ is equivalent to the entropy search acquisition function,
i.e. $\mathrm{EHIG}_t(x; \ell, \Ac) = \mathrm{ES}_t(x)$.
\end{prop}

\vspace{-2mm}
\begin{proof}[Proof of Proposition~\ref{prop:equiv_es}]
See Appendix~\ref{sec:app-proof-es}.
\end{proof}

\vspace{-2mm}
\paragraph{Decision-theoretic acquisition functions}
We next describe how the $\mathrm{EHIG}$ generalizes decision-theoretic
acquisition functions such as knowledge gradient (KG) and
expected improvement (EI).
Since these acquisition functions are often motivated from a 
perspective of a terminal decision, it is straightforward to show
how they are a special case of the $\mathrm{EHIG}$.
However, the choice of $\Ac$ and $\ell$ here
is insightful to review before extending $\mathrm{EHIG}$ to other scenarios.

First, the KG acquisition function can be written
\begin{align*}
    \text{KG}_t(x) = \expecf{p(y_x \mid \Dc_t)}{\mu_{t+1}^*(x, y_x)} - \mu_t^*,
\end{align*}
where $\mu_t^* = \sup_{x' \in \Xc} \expecf{p(f \mid \Dc_t)}{f(x')}$ is the
max value of the posterior mean of $f$ given data $\Dc_t$, and
$\mu_{t+1}^*(x, y_x) = \sup_{x' \in \Xc}
\expecf{p(f \mid \Dc_t \cup \{(x, y_x)\})}{f(x')}$
is the  max value of the posterior mean,
given both data $\Dc_t$ and observation $(x, y_x)$.
Second, the EI acquisition function can be written
\begin{align*}
    \text{EI}_t(x) = \expecf{p(y_x \mid \Dc_t)}{\max(y_x - f_t^*, 0)},
\end{align*}
where $f_t^* = \max \{\hat{f}(x_i)\}_{i=1}^{t-1}$, for $x_i \in \Dc_t$
and $\hat{f}(x_i)$ is the posterior expected value of $f$ at $x_i$.
This definition is equal to
the standard formulation of EI in the noiseless setting
(i.e. when $y_x = \truefunction(x)$ for queried $x$)
and the \textit{plug-in} formulation of EI in the noisy setting,
when $y_x = \truefunction(x) + \epsilon$ \citep{picheny2013benchmark, brochu2010tutorial}.

To view these decision-theoretic acquisition functions as special 
cases of the $\text{EHIG}$, suppose that after BO is complete,
we aim to make a single guess $x^* \in \Ac$ for the maximizer of $f$,
and then incur a loss equal to the value of the function at $x^*$, i.e.
$\ell(f, x^*) = -f(x^*)$.
For KG, we let $\Ac = \Xc$, in which case the
Bayes action $a^*$ is the maximizer of the posterior mean,
and for EI, we let $\Ac = \{x_i\}_{i=1}^{t-1}$,
in which case $a^*$ is the best queried point.
We visualize these as gold vertical lines in Figure~\ref{fig:bayesactions} (center panels).
In Appendix~\ref{sec:app-proofs}, we prove this equivalence with $\textrm{EHIG}_t$,
for $\text{KG}_t$ and $\text{EI}_t$.
Additionally, in Appendix~\ref{sec:app-proof-pi} we discuss connections to the probability of improvement (PI) acquisition function.

We thus see two key differences here, in comparison with information-based BO:
(i) the terminal action $a^*$ is a point estimate of the optimizer $x^*$
rather than a distribution over $\Xc$, and (ii) the loss does not depend on the particular value
of the true optima $x^*$
(nor on how accurately $a^*$ provides an estimate of $x^*$),
but rather only depends on the function value of the terminal action, $f(a^*)$.

\begin{figure*}[t]
\centering
\includegraphics[width=0.242\linewidth]{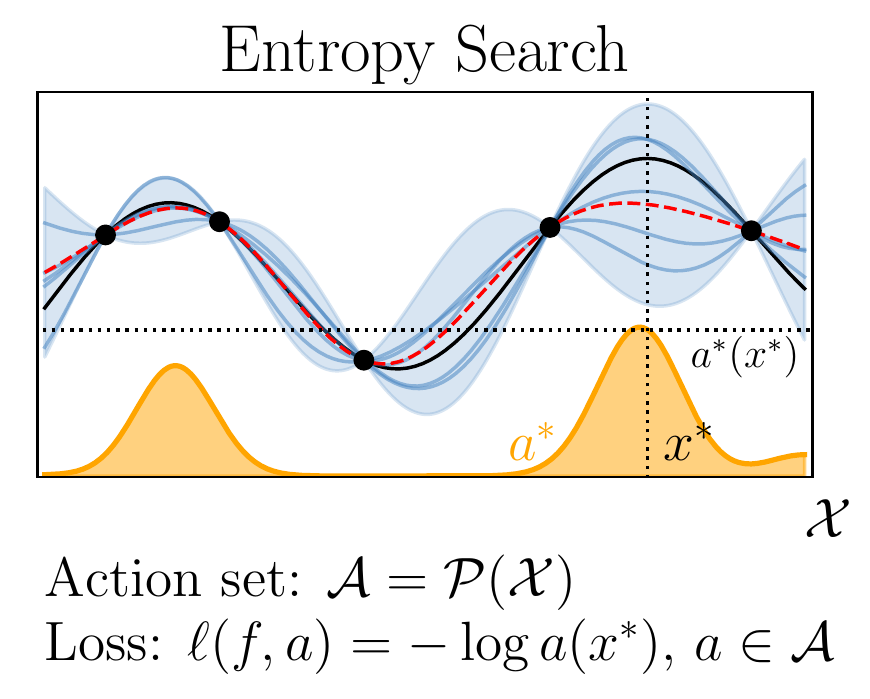}
\includegraphics[width=0.242\linewidth]{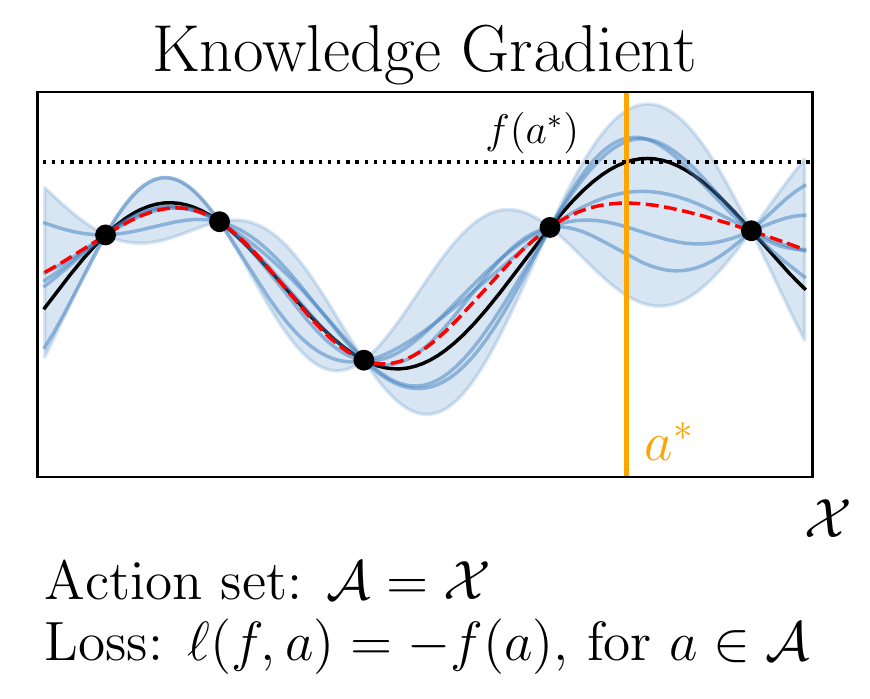}
\includegraphics[width=0.242\linewidth]{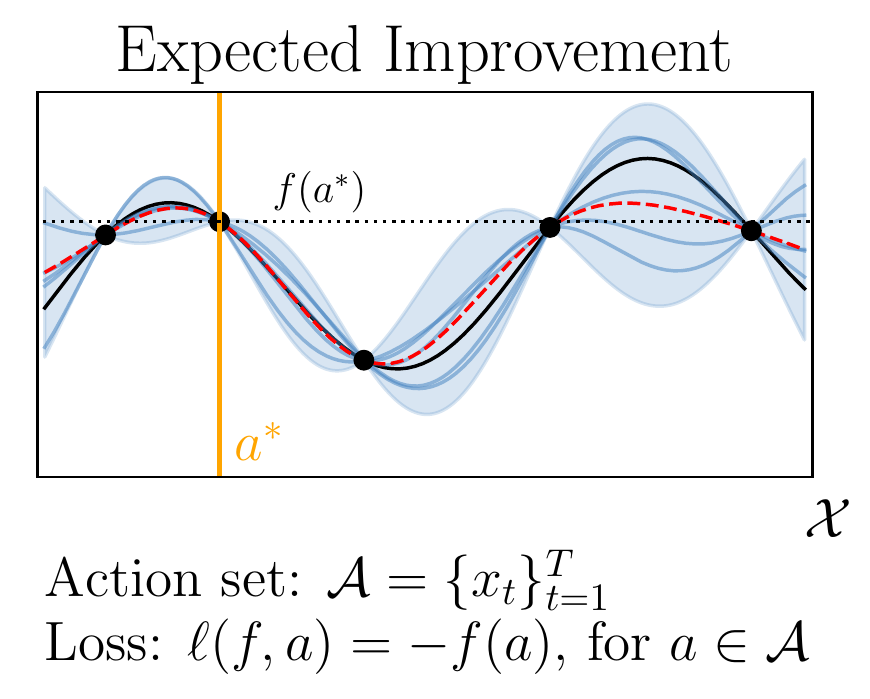}
\includegraphics[width=0.251\linewidth]{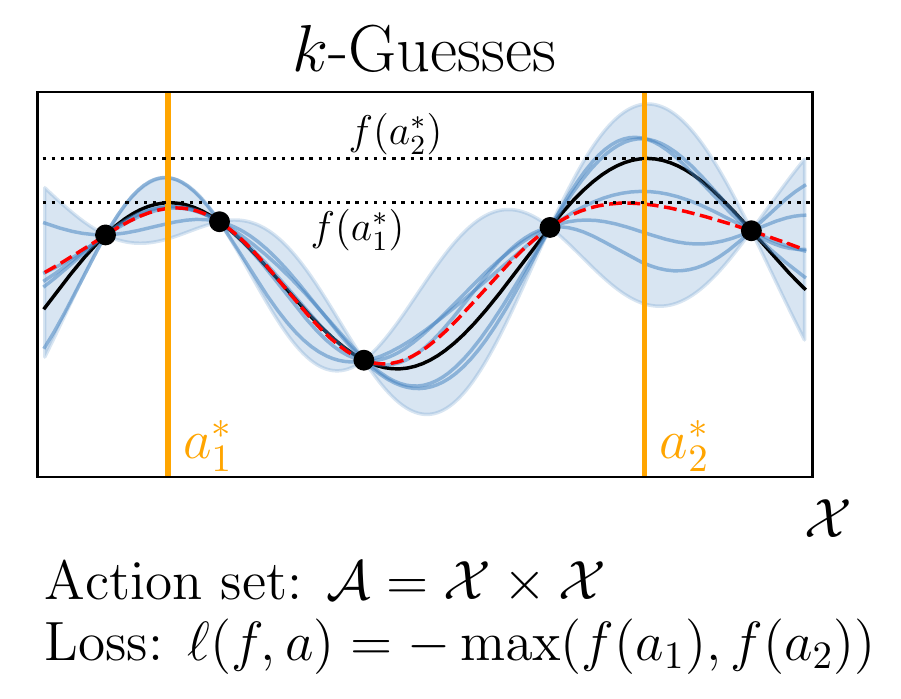}\\
\vspace{-1mm}
\caption{
\small
Example acquisition functions, and their corresponding Bayes actions
$a^*$ visualized. For each, we write the associated action set $\Ac$ and
loss function $\ell$ below the plot. In each plot, the true function is
a solid black line, the posterior mean is a red dashed line, the
observed data are black dots, and the Bayes action is shown in gold.
See Section~\ref{sec:existingacqfunctions} for further discussion.
}
\label{fig:bayesactions}
\vspace{-3mm}
\end{figure*}

\section{A Framework to Derive New Acquisition Functions for Custom Tasks}
\label{sec:novelacqfunctions}

\everypar{\looseness=-1}

There are many real-world problems that go beyond simple black-box optimization,
which have not been explicitly studied in the literature, and for which there does not exist a suitable acquisition function.
For these use cases, we can define an action set and loss based on details of the problem,
and use the $\textrm{EHIG}$ to provide a corresponding problem-specific acquisiton function.
As examples of this, below we apply the $\textrm{EHIG}$ to a number of relevant
problems where, as far as we are aware, no corresponding acquisition function has
been developed by prior work.

\vspace{-1mm}
\paragraph{Illustrative example: $k$-guesses}

As a simple example to illustrate our framework,
suppose, after optimization is complete, we are allowed
to make a batch of $k$ guesses for the function maximizer $x^*$,
and then recieve a final reward based on the best guess.
This setup appears in cases where, after BO is complete,
we can make a \textit{batch} of final designs
(e.g. synthesize a final set of materials \cite{terayama2021black}
or train a final set of models \cite{wu2016parallel, snoek2012practical}),
and only care about the single \textit{best} design of the batch.
We can thus view the action set as
$\Ac = \Xc^k$, and loss as
$\ell(f, a) = - \max \left( f(a_1), \ldots, f(a_k) \right)$.

Figure~\ref{fig:bayesactions} (far right) provides a visualization of this for $k=2$.
In this scenario, one of the Bayes actions ($a_2^*$) is near the
maximizer of the posterior mean (similar to KG), while
the other ($a_1^*$) is separated from the first.
Intuitively, to minimize the expected loss, the second guess should have
both a high posterior mean, and also a low correlation
to the first guess---the first guess has a better chance of a
low loss, but in cases where it fails, we want the second
guess to \textit{not} fail (i.e. not match the first guess),
while also achieving a low loss.
In practice, this yields an $\textrm{EHIG}$ acquisition function that
has a similar but distinct exploration strategy from KG.
It spends a small portion of the budget on queries that give some 
information about not just the first guess, but the other $k$-$1$ guesses as well.

\vspace{-1mm}
\paragraph{Top-$k$ optimization with diversity}
Instead of a single
optimal point in $\Xc$,
there are applications where we wish to estimate a set of top-$k$ optima, i.e. the subset of
$\Xc$ of size $k$ that has the highest sum of values under $f$.
Examples of this can be found in materials discovery \citep{liu2017materials, terayama2021black},
sensor networks \citep{abbasi2008mote, garnett2010bayesian}, and medicine \citep{xie2018diversity}.
Note that the goal of this problem is distinct from the $k$-guesses
example described above.
When $\Xc$ is continuous, to avoid redundant solutions, we may wish to
carry out the task of \textit{top-$k$ optimization with diversity}, which aims to
find the top-$k$ optima such that $d(x_i, x_j) \geq c$, 
$\forall i, j \in \{1, \ldots, k\}$, for a problem-specific distance $d$.
As one example of our $\mathrm{EHIG}$ framework, suppose that we choose 
$\Ac = \Xc^k$ (where $a = (a_1, \ldots, a_k) \in \Ac$ denotes a set of top-$k$ points)
and incur the loss 
\begin{align}
    \label{eq:topkloss}
    \ell(f, a) = - \sum_i f(a_i) - \sum_{1 \leq i < j \leq k} d(a_i, a_j).
\end{align}
Note that we can select the distance function $d$ here to match the problem-specific
constraint.
Intuitively speaking, this choice of $(\ell, \Ac)$ yields an EHIG acquisition function that
makes a sequence of queries which rotate focus between multiple diverse optimal points in the domain.

\vspace{-1mm}
\paragraph{Multi-level set estimation}
The goal of level set estimation (LSE) is to estimate a subset of the 
 design space $\Xc$, where function values are larger than
 a given threshold $c$, $\Sc_c = \{x \in \Xc : f(x) > c\}$.
This task appears in a number of applications, including
catalyst design \citep{zhong2020accelerated},
interactive learning \citep{boecking2020interactive},
and environmental monitoring \citep{singh2008nonparametric}.
While many prior works have studied standard LSE, here we consider the task of 
multi-level set estimation (MLSE), where we are given $m$ thresholds satisfying $c_1 < \ldots < c_m$ and
want to estimate $m+1$ sets:
$\Sc_{i} = \{x \in \Xc : c_i < f(x) < c_{i+1}\}$
for $i=\{0, \ldots, m\}$ (where $c_0 := -\infty$ and $c_{m+1} := +\infty$).
This is useful in the above applications when we have more than one
threshold of interest---for example, public health policy makers must estimate regions where disease
prevalence exceeds 1\%, 2\%, etc., for graded reopening decisions \cite{oh2021design, yiannoutsos2021bayesian}.

As one approach to MLSE using the $\textrm{EHIG}$,
we focus on settings with a discrete set of design points
$\Xc_0 \subset \Xc$, $|\Xc_0| = J$ \citep{gotovos2013active, kandasamy2019myopic}.
Suppose, after querying is complete, we must choose a set of values
$a \in \Ac = [0,1]^{J \times m}$ (one for each $x \in \Xc_0$ and $i \in \{0, \ldots, m\}$),
which represent level set identity variables.
Suppose we then incur a loss with the following form, that depends on these identity variables $a$,
as well as on a flexible relation $r(f(x), c_i)$ between function values $f(x)$ and thresholds $c_i$, i.e.
\begin{align}
    \ell(f, a(x)) = - \sum_{i=1}^m \sum_{x \in \Xc_0}  a_i(x) r( f(x), c_i).
    \label{eq:mlseloss}
\end{align}
For instance, if $r(f(x), c_i) = f(x) - c_i$, then
the optimal $a_i(x)$ should specify the $c_i$-super level set for each
$i \in \{1, \ldots, m\}$, i.e. $a_i(x)=1$
for each $x \in \Xc_0$ with $f(x) > c_i$ and $a_i(x)=0$ otherwise.
This example loss yields an acquisition function that, empirically, focuses samples around the boundaries
of multiple level sets of a black-box function simultaneously.

\vspace{-1mm}
\paragraph{Sequence search}
We define \textit{sequence search} as the task of estimating a 
sequence of inputs $(x_1, \ldots, x_m) \in \Xc^m$
with outputs values matching a set of problem-specific criteria.
For example, we may wish to estimate a sequence of inputs corresponding to a
predefined set of function values
$(y^\circledast_1, \ldots, y^\circledast_m)$.
This finds applications in materials science, such as in the task of
synthesizing a nanoparticle library \citep{fong2021utilization} (i.e. finding a
set of input conditions that yield a set of nanoparticles of different pre-defined sizes).
As another example, in the context of public health, we may be interested in a set of
locations where vaccination rates equal some pre-specified values
(e.g. $(20\%, \ldots, 80\%)$) when making decisions involving vaccine allocations,
as we describe in Section~\ref{sec:experiments}.
As an example of these applications in our $\textrm{EHIG}$ framework,
we might have the action set $\Ac = \Xc^m$ and loss 
\vspace{-0.7mm}
\begin{align}
    \ell(f, \bfa) = \sum_{i=1}^m (f(\bfa_m) - y^\circledast_m)^2.
    \label{eq:mvsloss}
\end{align}
These examples all aim to show that the $\textrm{EHIG}$ can be used to define a problem-specific
acquisition function, which can be tailored to the details of a particular use case.
As a result, when used in Algorithm~\ref{alg:hes}, we gain a customizable optimization framework that can
be applied to a variety of novel problem settings with special-purpose losses.
\section{Gradient-based Acquisition Optimization}
\label{sec:acqopt}

\everypar{\looseness=-1}

At each iteration of \textsc{$H_{\ell, \Ac}$-Entropy Search} (Algorithm~\ref{alg:hes}),
we optimize the acquisition function to select the next query
$x_t = \argmax_{x \in \Xc} \text{EHIG}_t(x; \ell, \Ac)$.
Classically, zeroth order optimization routinues have been used
for acquisition optimization in BO. However, recent work
has developed gradient-based methods for optimizing certain
acquisition functions \citep{wilson2018maximizing, balandat2020botorch},
which can allow for efficient acquisition optimization over $\Xc$.
We work on similar methodology here---namely, we develop a
gradient-based acquisition optimization procedure for
appropriate settings (i.e. assuming continuous $\Xc$ and $\Ac$, and certain conditions on $\ell$).
We can, for example, apply this gradient-based optimization to each of the acquisition functions
described in Section~\ref{sec:novelacqfunctions}, for which we show experimental
results in Section~\ref{sec:experiments}.

Similar to related work \citep{wilson2018maximizing, balandat2020botorch},
we give the following derivation with a focus on Gaussian process (GP) models,
though the methodology
can be extended to other models in which we can apply the reparameterization
procedure described below to differentiate through posterior model parameters.

\paragraph{Differentiable loss function}
We first describe a few assumptions that must be satisfied to carry out the
gradient-based optimization procedure.
Denote the posterior expected loss given $\Dc$ by
$L(\Dc, a) := \expecf{p(f \mid \Dc)}{\ell(f, a)}$.
We assume that this loss function depends only on the function value of
$f$ at a finite number of points, i.e. there exists functions
$\tx_1(a), \cdots, \tx_K(a)$, and a function $\ell': \Rbb^K \times \Ac \to \Rbb$,
for $K \in \Nbb$, such that  
\begin{align}
    \ell(f, a) = \ell'(f(\tx_1(a)), f(\tx_2(a)), \cdots, f(\tx_K(a)), a).
    \label{eq:loss_function_requirement}
\end{align}
This requirement is satisfied by the loss functions in
Section~\ref{sec:novelacqfunctions}. 
For brevity, denote the sequence $\tx_1(a), \cdots, \tx_K(a)$ by $\tx_{1:K}(a)$
and $f(\tx_1(a)), \cdots, f(\tx_K(a))$ by $f(\tx_{1:K}(a))$. 
We assume that the functions $\tx_k$ and $\ell'$ are differentiable with respect to all arguments.
Given a dataset $\Dc$ and GP prior, the posterior distribution of $f(\tx_K(a))$ is also
Gaussian. In particular, there exist functions
\begin{align*} 
    \mu: \tx_{1:K}(a) \times \Dc \mapsto \Rbb^K
    \hspace{2mm}\text{and}\hspace{2mm}
    U: \tx_{1:K(a)}  \times \Dc \mapsto \Rbb^{K \times K}
\end{align*} 
such that
$f(\tx_{1:K}(a)) = \mu(\tx_{1:K}(a); \Dc) +  U(\tx_{1:K}(a); \Dc) \epsilon$
where $\epsilon$ is drawn from a $K$-dimensional standard normal distribution.
We can combine the above results to get
\begin{align*}
    L(\Dc, a) = \expecf{\epsilon}{\ell'(\mu(\tx_{1:K}(a); \Dc) + U(\tx_{1:K}(a); \Dc) \epsilon, a) }. 
\end{align*}
A key property is that we can compute unbiased gradients of this
with respect to both $\Dc$ and $a$, as
\begin{align*}
    \nabla L(\Dc, a) = \expecf{\epsilon}{ \nabla \ell'(\mu(\tx_{1:K}(a); \Dc)
    + U(\tx_{1:K}(a); \Dc) \epsilon, a) }. 
\end{align*}

\vspace{-1mm}
\paragraph{Differentiable acquisition function}
For a given input $x \in \Xc$, let $y(x, \Dc)$ denote the posterior predictive distribution of our model.
Note that there exists a deterministic function $\bar{y}(x, \Dc, \lambda)$ such that
$y(x, \Dc) = \bar{y}(x, \Dc, \lambda)$, where $\lambda$ is drawn from a standard normal distribution. 
Hence, if $\ell$ satisfies \eqref{eq:loss_function_requirement},
then we can optimize $\mathrm{EHIG}_t$ with gradient descent. In particular, we can write
\begin{align}
    \inf_{x \in \Xc} -\mathrm{EHIG}_t(x; \ell, \Ac)
    = \inf_{x \in \Xc} \inf_{\ta: \lambda \mapsto \Ac}
        \expec{\lambda, \epsilon}{ \ell'(\hat{\mu}(x, \ta(\lambda))
        + \hat{U}(x, \ta(\lambda)) \epsilon, \ta(\lambda)) }
    \label{eq:very_long}
\end{align}
where in \eqref{eq:very_long}, to avoid clutter, we use the shorthand
$\hat{\mu}(x, \ta(\lambda)) := \mu(\tx_{1:K}(\ta(\lambda)); \Dc\cup \bar{y}(x, \Dc, \lambda))$,
and $\hat{U}(x, \ta(\lambda)) := U(\tx_{1:K}(\ta(\lambda)); \Dc\cup \bar{y}(x, \Dc, \lambda))$.
Importantly, we can compute the unbiased gradient
of the quantity $\expec{\lambda, \epsilon}{ \ell'(\hat{\mu}(x, \ta(\lambda))
+ \hat{U}(x, \ta(\lambda)) \epsilon, \ta(\lambda)) }$.
In practice, we can also take gradients of a Monte Carlo estimate of
\eqref{eq:very_long} \citep{balandat2020botorch}, by
fixing samples of $\lambda, \epsilon$ throughout the optimization.
Specifically, we can sample $\lambda_1, \cdots, \lambda_M$ and $\epsilon_1, \cdots, \epsilon_N$
and approximate \eqref{eq:very_long} via
\begin{align}
    \inf_{x \in \Xc} -\mathrm{EHIG}_t(x; \ell, \Ac)
    \approx
    \inf_{x \in \Xc} \inf_{a_1, \ldots, a_M}
    \frac{1}{NM} \sum_{m, n} \ell'(\hat{\mu}(x, a_m) + \hat{U}(x, a_m) \epsilon_n, a_m),
\end{align}
where we use $a_m = a(\lambda_m)$ for brevity.
Under the assumptions above, we can compute the unbiased gradient of this quantity, and 
by using systems such as GPyTorch \citep{gardner2018gpytorch} and BoTorch \citep{balandat2020botorch}
we can compute this gradient efficiently via automatic differentiation.
\section{Experiments}
\label{sec:experiments}

\everypar{\looseness=-1}

We evaluate our proposed method on the example tasks described
in Section~\ref{sec:novelacqfunctions}:
top-$k$ optimization with diversity, multi-level set estimation, and sequence search.
For these applications, we show comparisons against a set of baselines on real and synthetic black-box functions.

\begin{figure*}[t]
\centering
\includegraphics[width=0.325\linewidth]{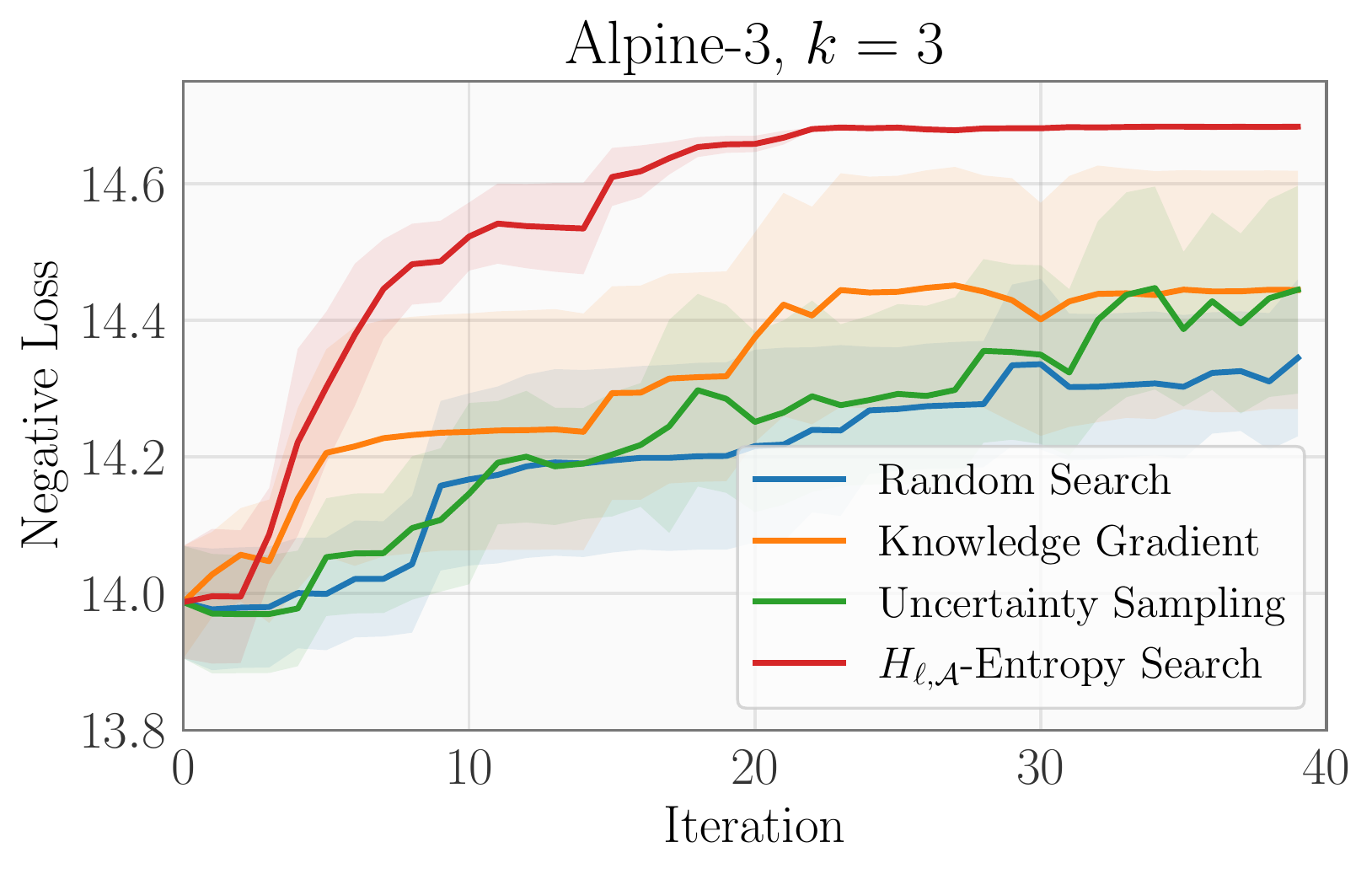}
\includegraphics[width=0.325\linewidth]{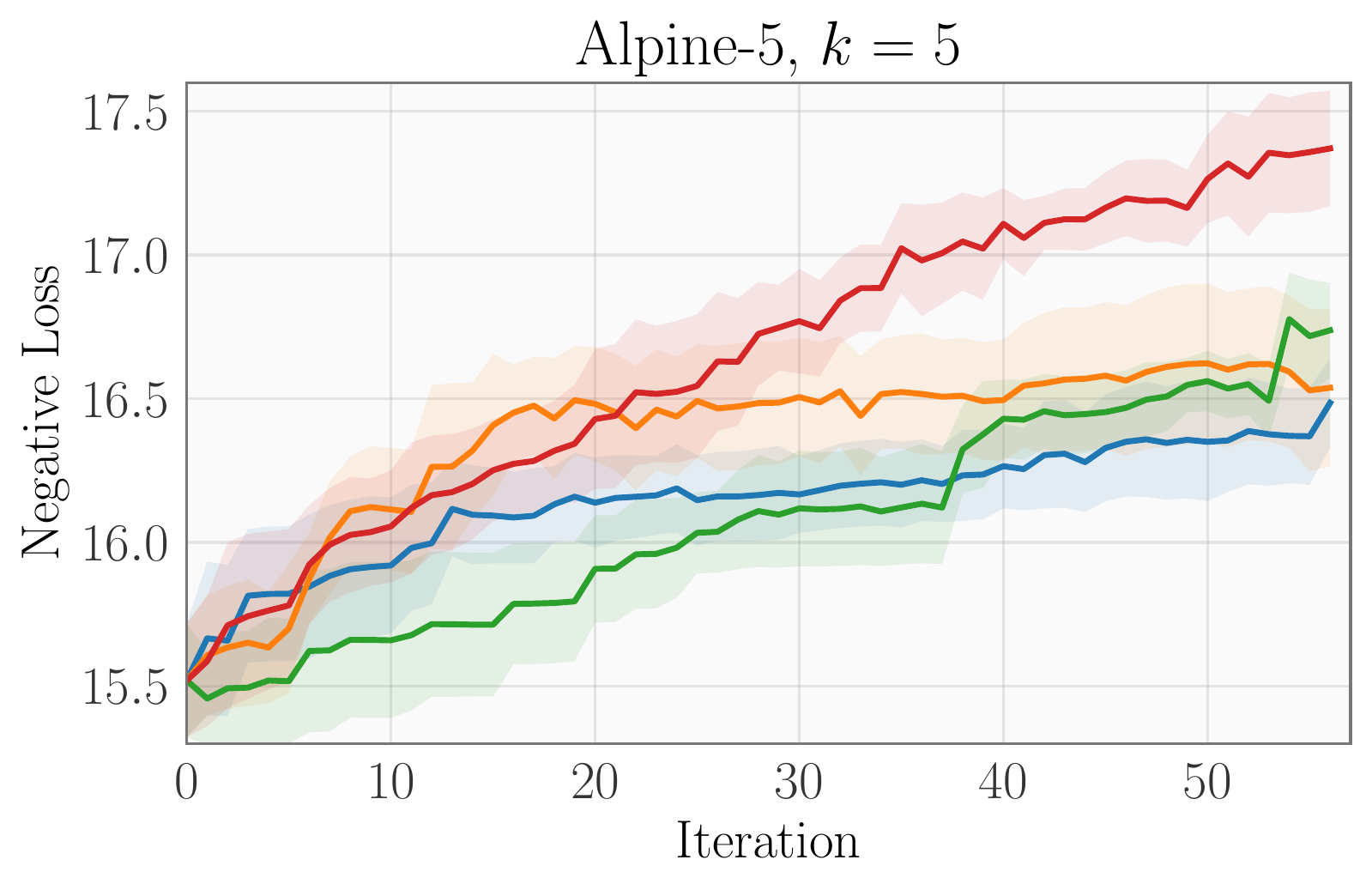}
\includegraphics[width=0.325\linewidth]{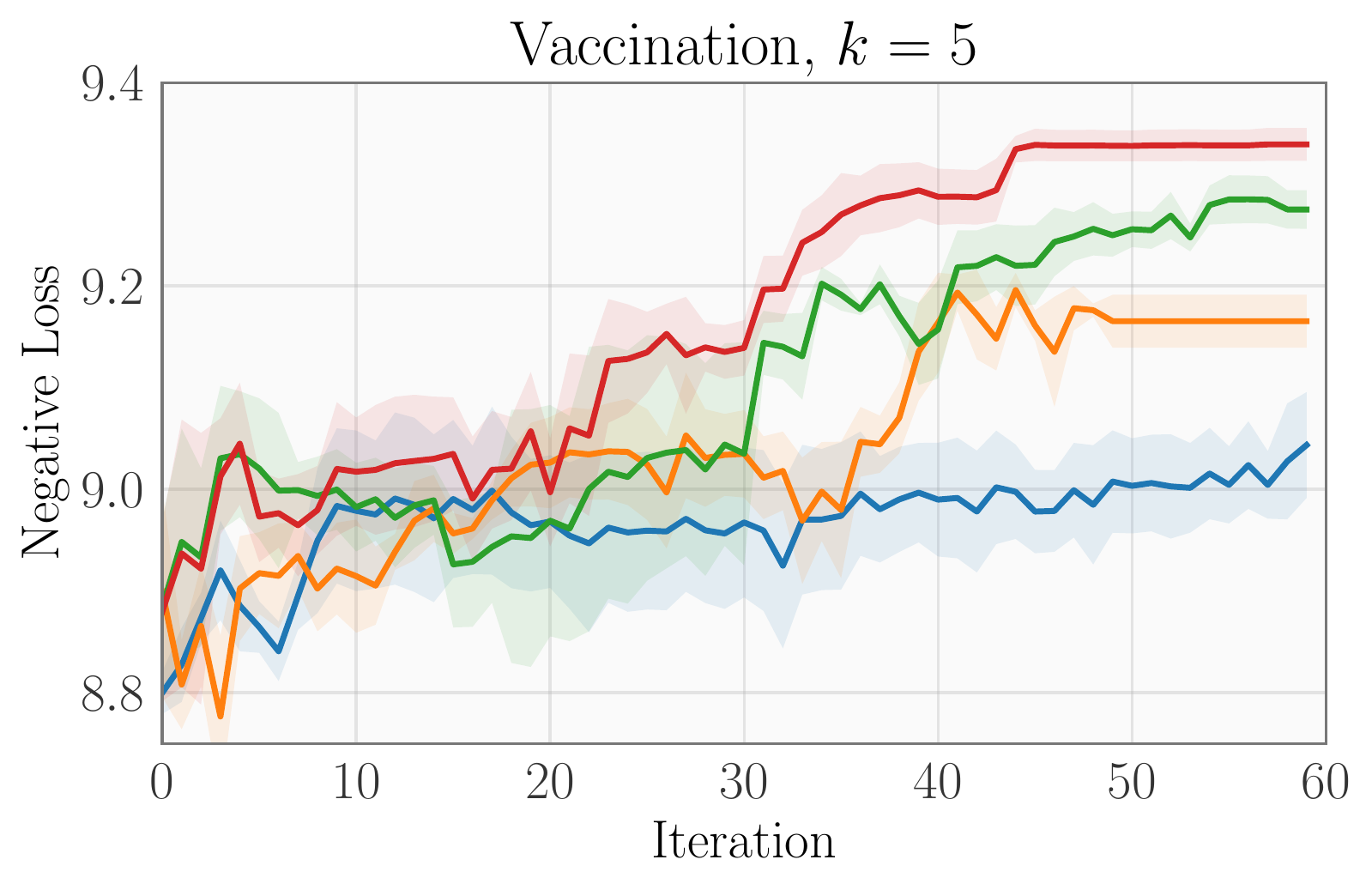}
\mbox{\hspace{1mm}}
\includegraphics[width=0.245\linewidth]{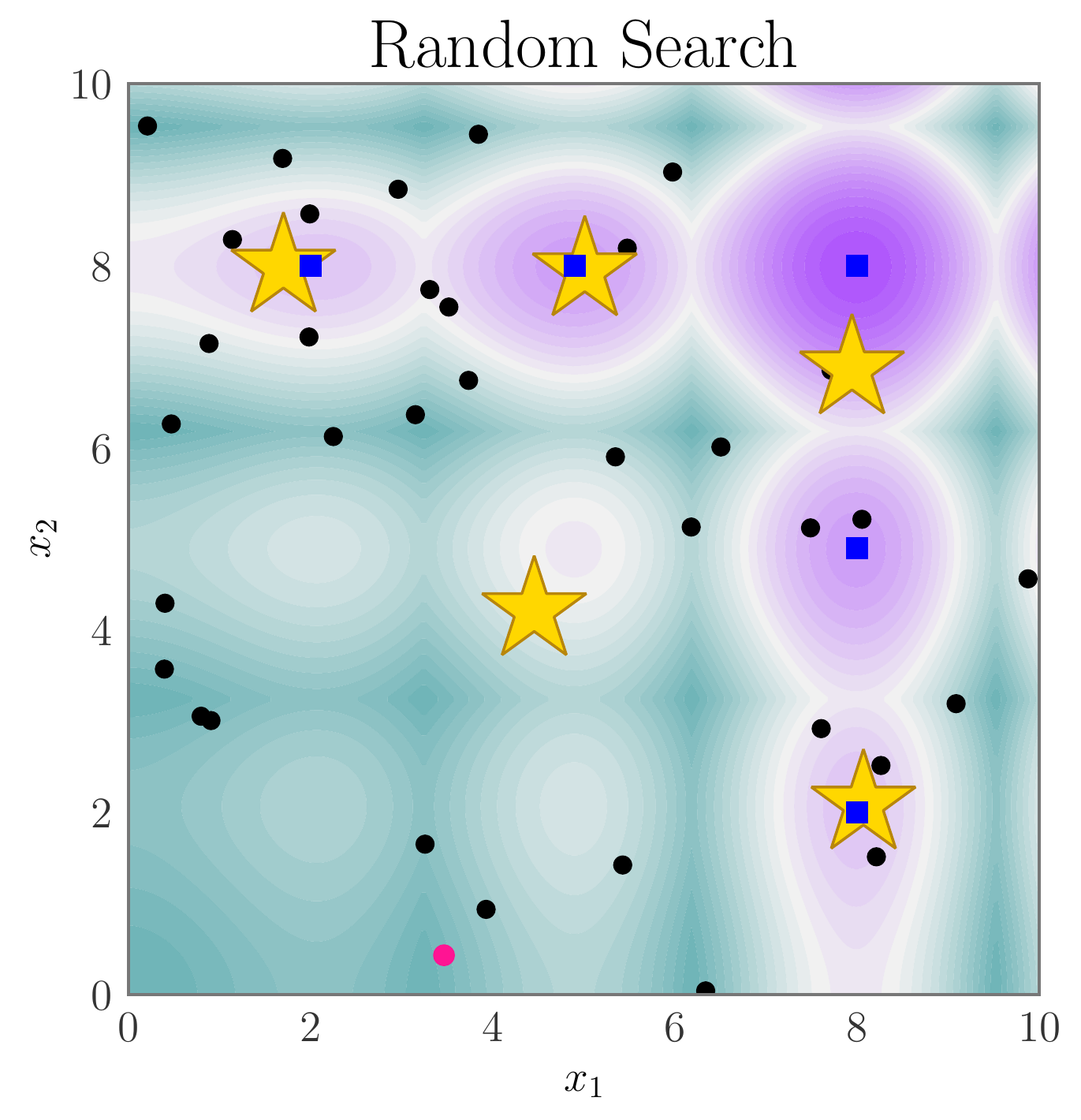}
\includegraphics[width=0.245\linewidth]{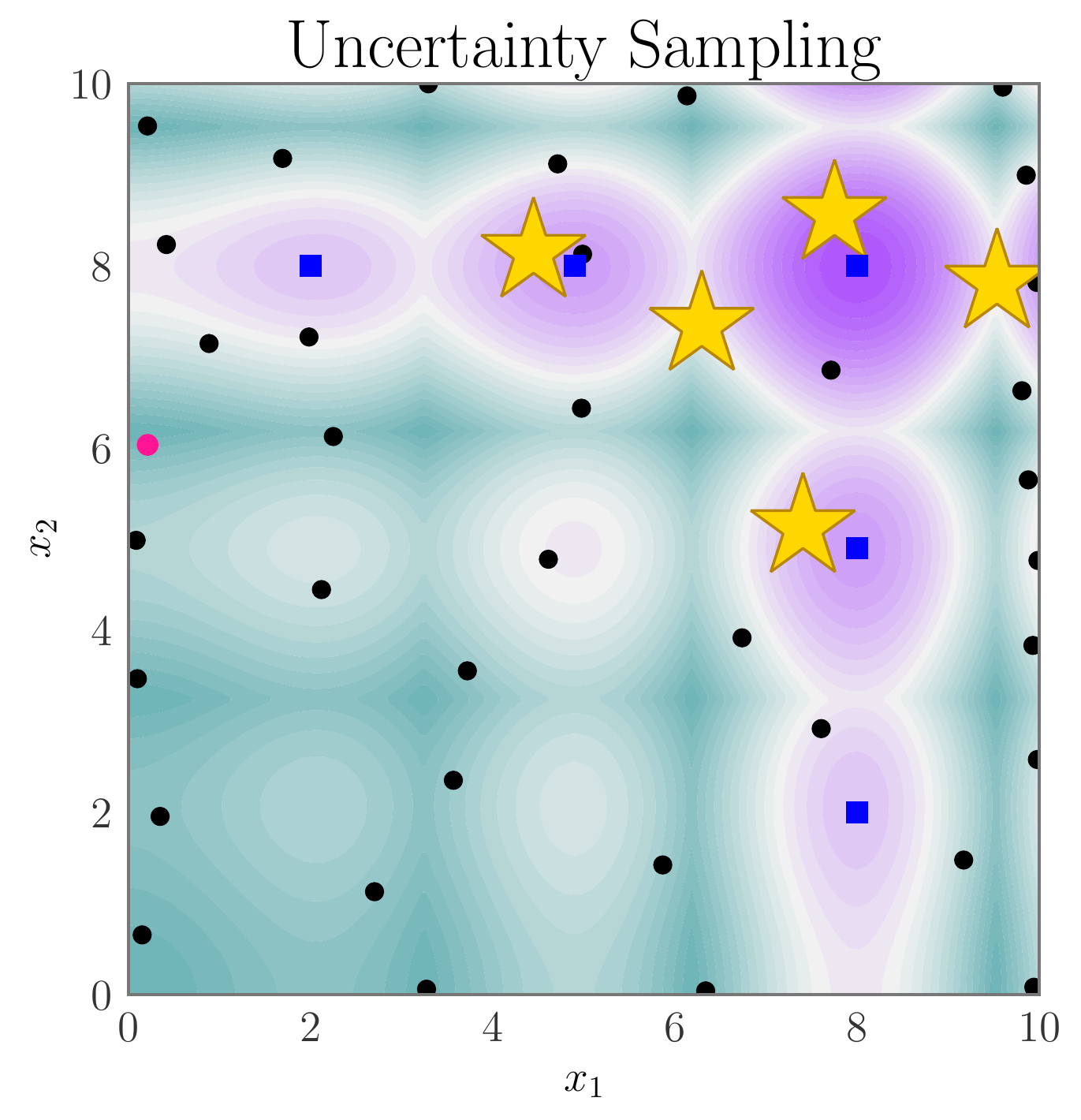}
\includegraphics[width=0.245\linewidth]{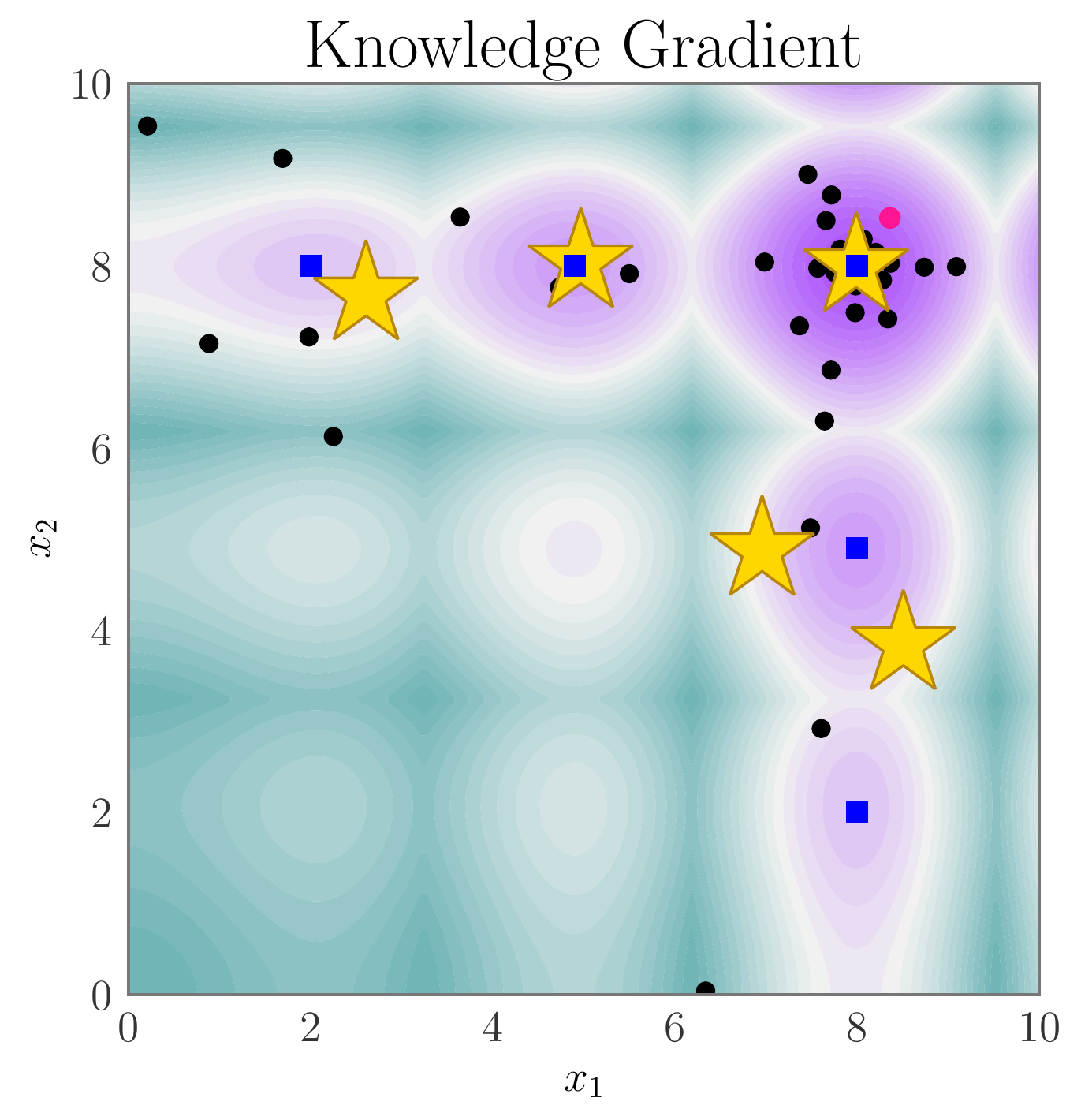}
\includegraphics[width=0.245\linewidth]{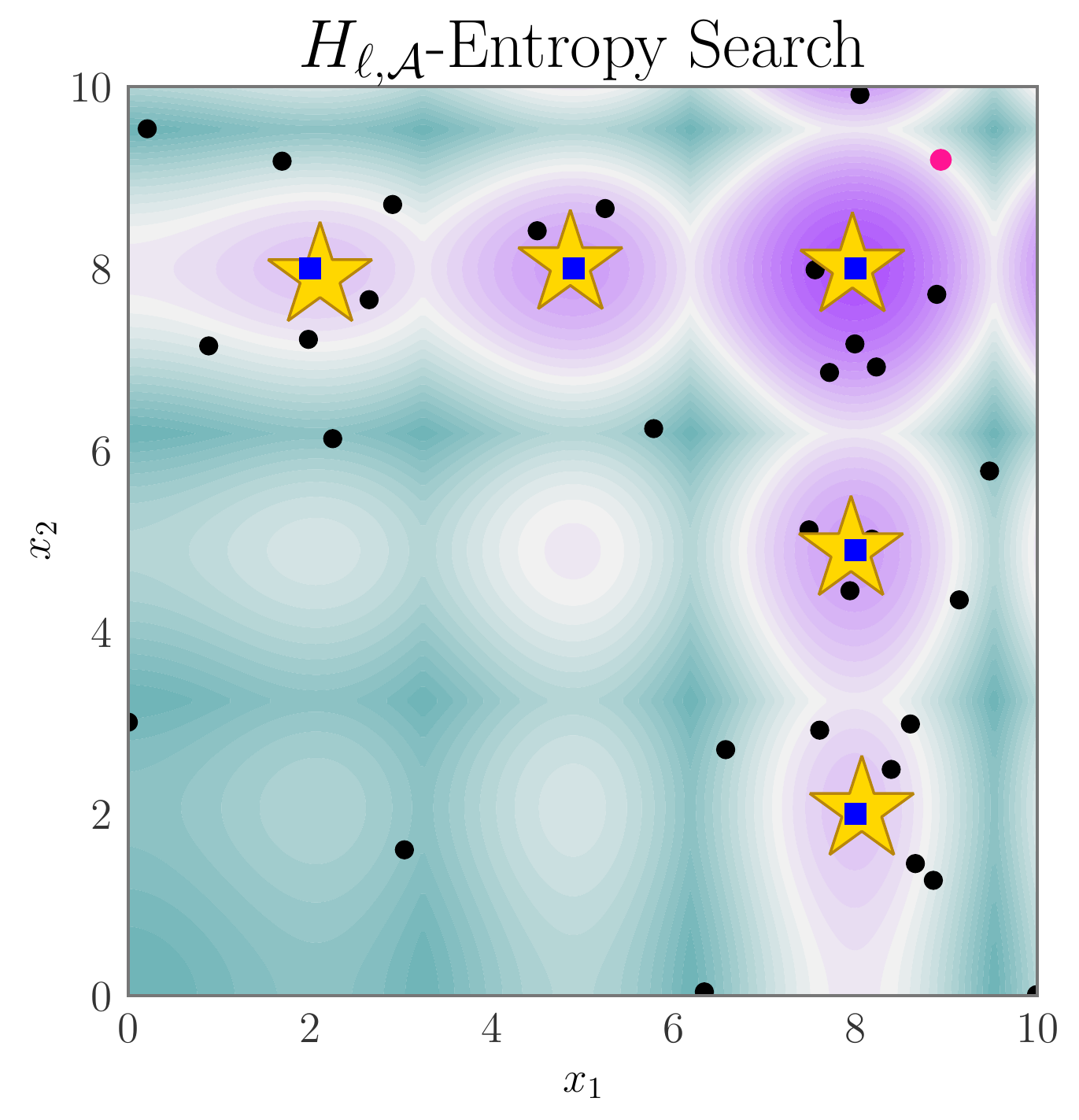}\\
\vspace{-2mm}
\caption{
\small
\textbf{Top-$k$ optimization with diversity.}
\textit{Top row:} Plots of the negative loss $-\ell(\truefunction, a^*)$ versus iteration
for all methods, on the \textit{Alpine-3}, \textit{Alpine-5} and
\textit{Vaccination} functions, where error bars
represent one standard error.
\textit{Bottom row:} Visualization of methods in two dimensions,
showing the set of ground-truth top-$5$ diverse design points (blue squares),
queries $\Dc_t$ taken (black dots), acquisition function optimizer (pink dot),
and the estimated set of top-$5$ diverse design points (gold stars).
}
\label{fig:topk}
\vspace{-5mm}
\end{figure*}

\textbf{Comparison methods.} In our experiments, we compare the
following set of acquisition strategies:
\begin{itemize}[parsep=0pt, topsep=0pt, itemsep=3pt, leftmargin=4mm]
    \item \textsc{$H_{\ell, \Ac}$-Entropy Search (HES).} We follow Algorithm~\ref{alg:hes}, using the loss
    and action set for each task as described in Section~\ref{sec:novelacqfunctions}, and the
    gradient-based procedure outlined in Section~\ref{sec:acqopt}.
    \item \textsc{Random Search (RS).} At each iteration, we draw a sample $x_t$ uniformly at
    random from $\Xc$. 
    \item \textsc{Uncertainty Sampling (US).} At each iteration, we select
    the point that maximizes the posterior predictive variance, i.e.
    $x_t = \argmax_{x \in \Xc} \mathrm{Var}[p(y_x \mid \Dc_t)]$.
    \item \textsc{Knowledge Gradient (KG).} We show KG as a representative
    method for standard BO. KG allows us to carry out a similar gradient-based
    procedure as in \textsc{HES}.
    \item \textsc{Probability of Misclassification (POM).} This is a common acquisition function
    for level set estimation \citep{bryan2005active}.
    We predict whether a point is above a threshold, represented by a binary variable $z$, and select the design with maximal label uncertainty $x_t  = \argmin_{x \in \Xc} \max_{z \in \{0, 1\}} p(z|x)$.
\end{itemize}
\vspace{-1mm}
Note that we are restricted to comparing against relatively general-purpose baseline methods,
as more-specific acquisition functions have not previously been developed for the tasks below.

\vspace{-3mm}
\paragraph{Top-$k$ Optimization with Diversity}
In our first task, the goal is to find a set of $k$ diverse elements in $\Xc$, each with a 
high value of $\truefunction$.
To assess each method, at each iteration we record the negative loss $-\ell(\truefunction, a^*)$
using  \eqref{eq:topkloss}---i.e. the
\textit{negative top-$k$ with diversity loss} of the Bayes action
$a^* = \arg\inf_{a \in \Ac} \expecf{p(f \mid \Dc_t)}{ \ell(f, a) }$
on the  true function $\truefunction$---using the set of queries $\Dc_t$
produced by the given method.
Intuitively, if a method makes a set of queries that yield a good estimate
of diverse top-$k$ elements, 
it will score a high value under this metric.

In Figure~\ref{fig:topk} (\textit{bottom row}) we visualize results on the multimodal
\textit{Alpine-$d$} benchmark function (see appendix for details).
Here, we can see that HES concentrates queries over five local optima of this function,
while KG allocates a majority of samples on only the highest peak, and both US 
and RS distribute their queries over the full domain $\Xc$.
In Figure~\ref{fig:topk} (\textit{top row}), we compare performance by
plotting the negative loss versus iteration on two higher dimensional examples.
We also compare each method on the \textit{Vaccination} function (provided by \cite{yuan2021mobility}),
which returns the vaccination rate for locations in the continential United States, given
an input \textit{(latitude, longitude)}.
The goal of this task is to efficiently find a set of five diverse locations
with highest vaccination rates.
We show results in Figure~\ref{fig:topk} (\textit{top row, right}), and see a similar
advantage of HES over comparison methods.

\begin{figure*}[t]
\centering
\includegraphics[width=0.325\linewidth]{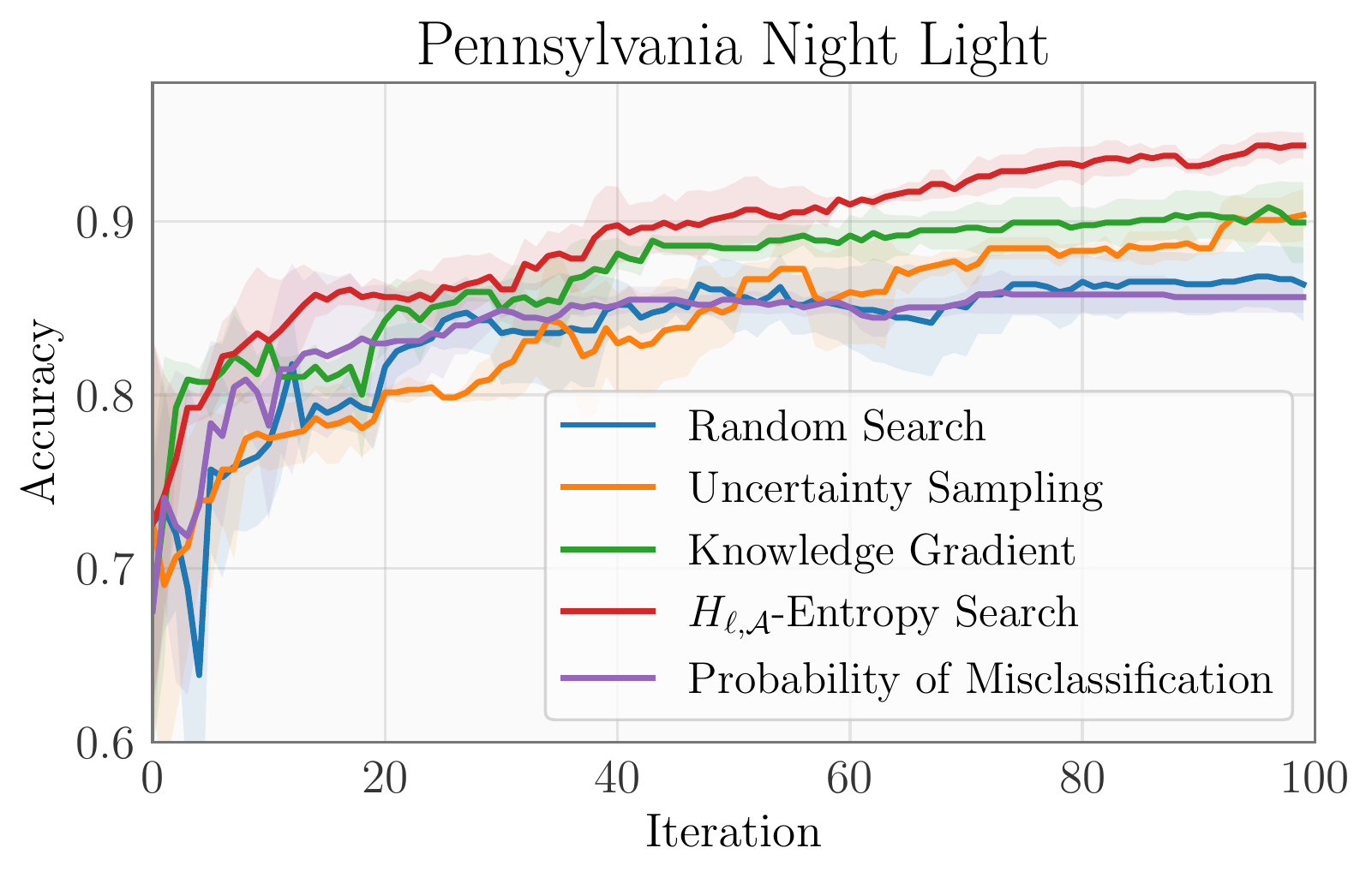}
\includegraphics[width=0.325\linewidth]{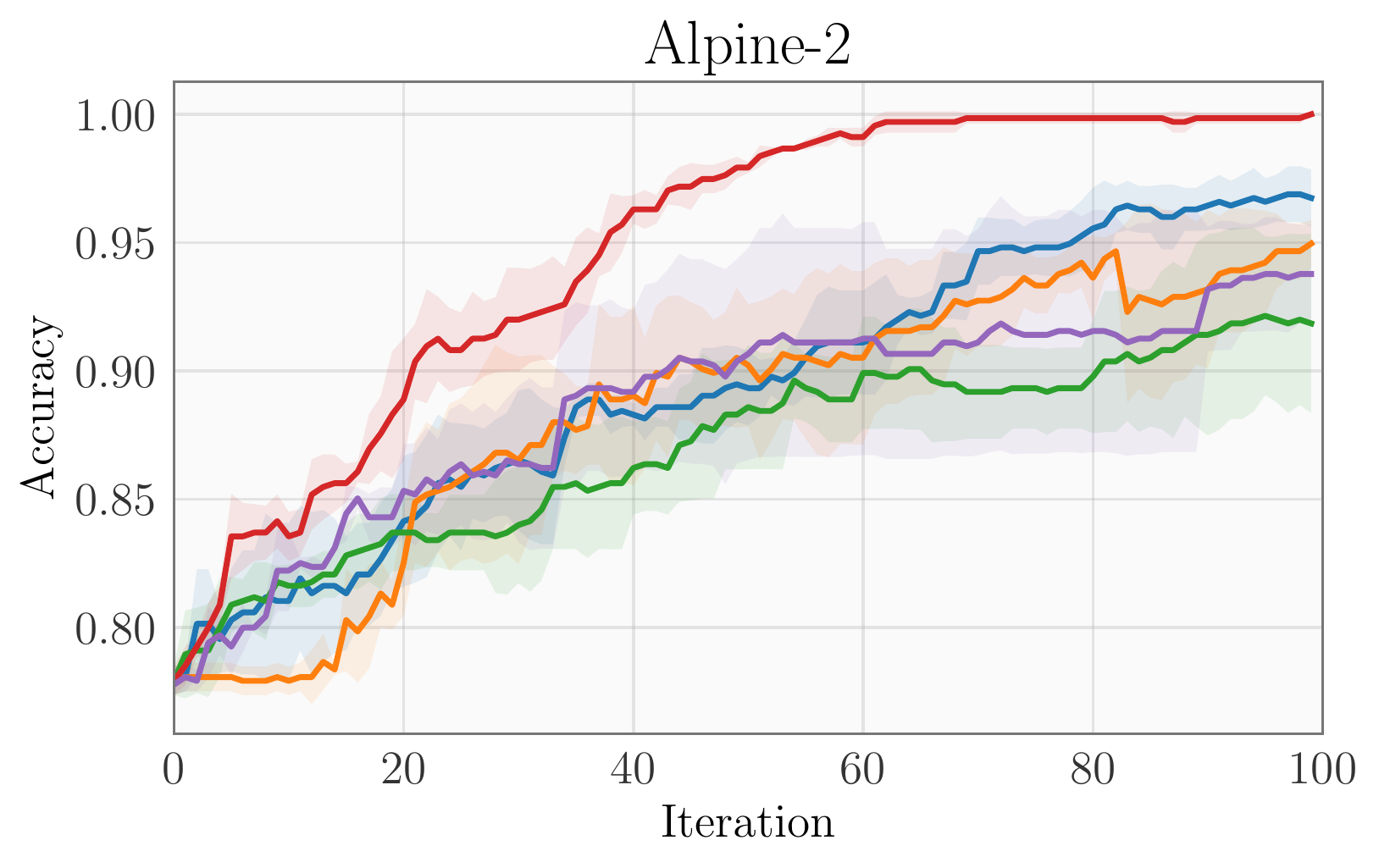}
\includegraphics[width=0.325\linewidth]{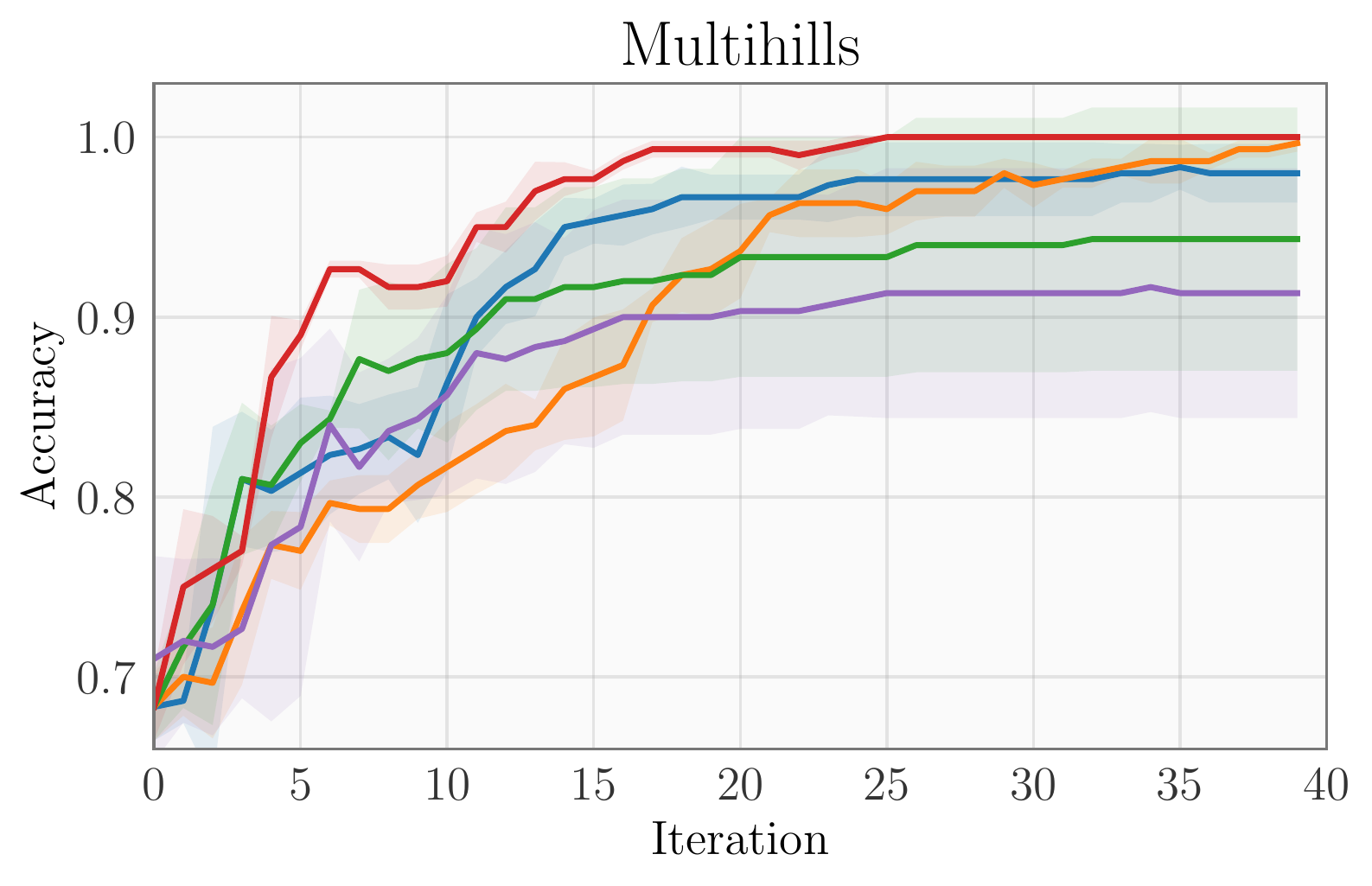}\\
\includegraphics[width=0.245\linewidth]{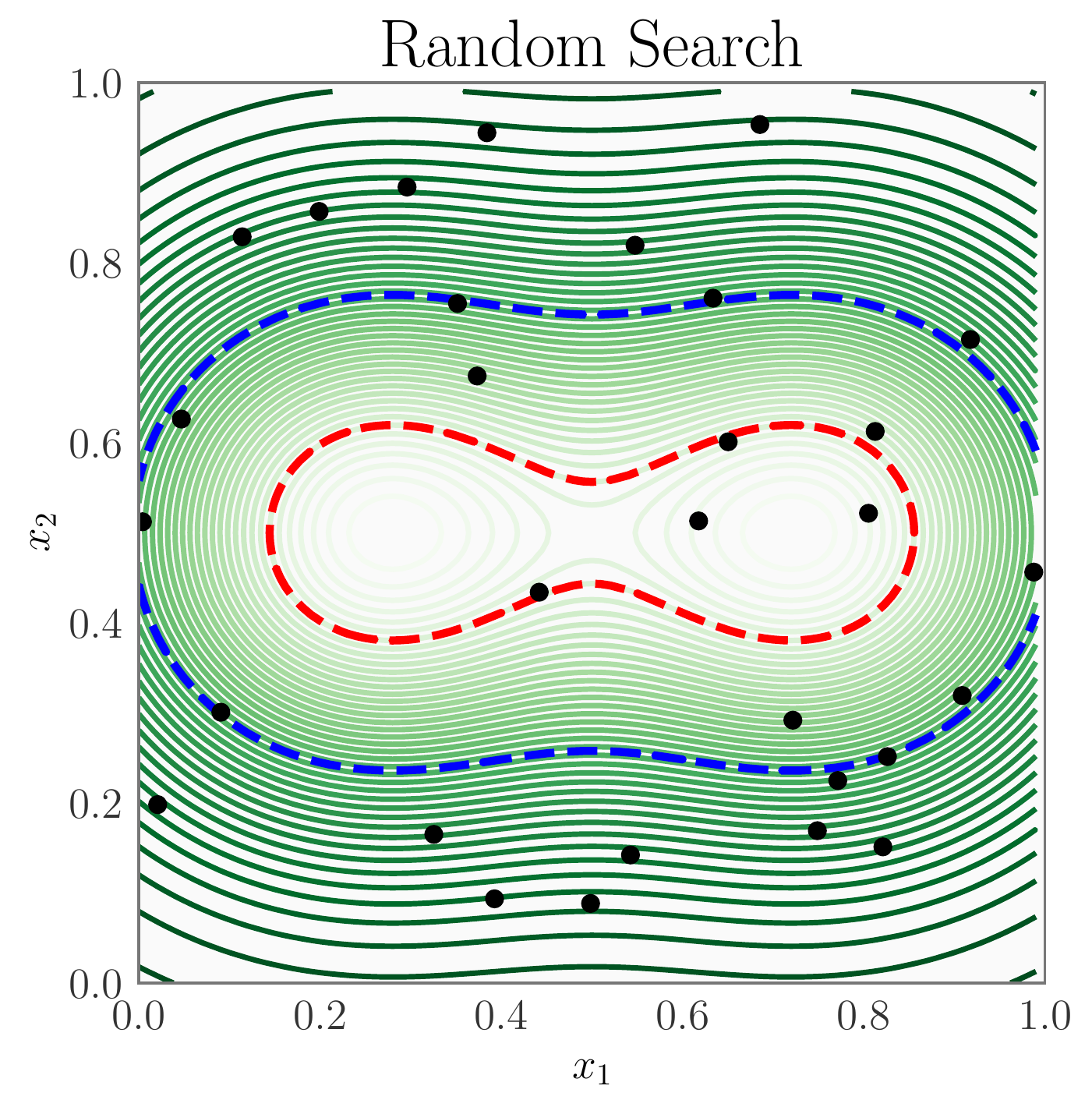}
\includegraphics[width=0.245\linewidth]{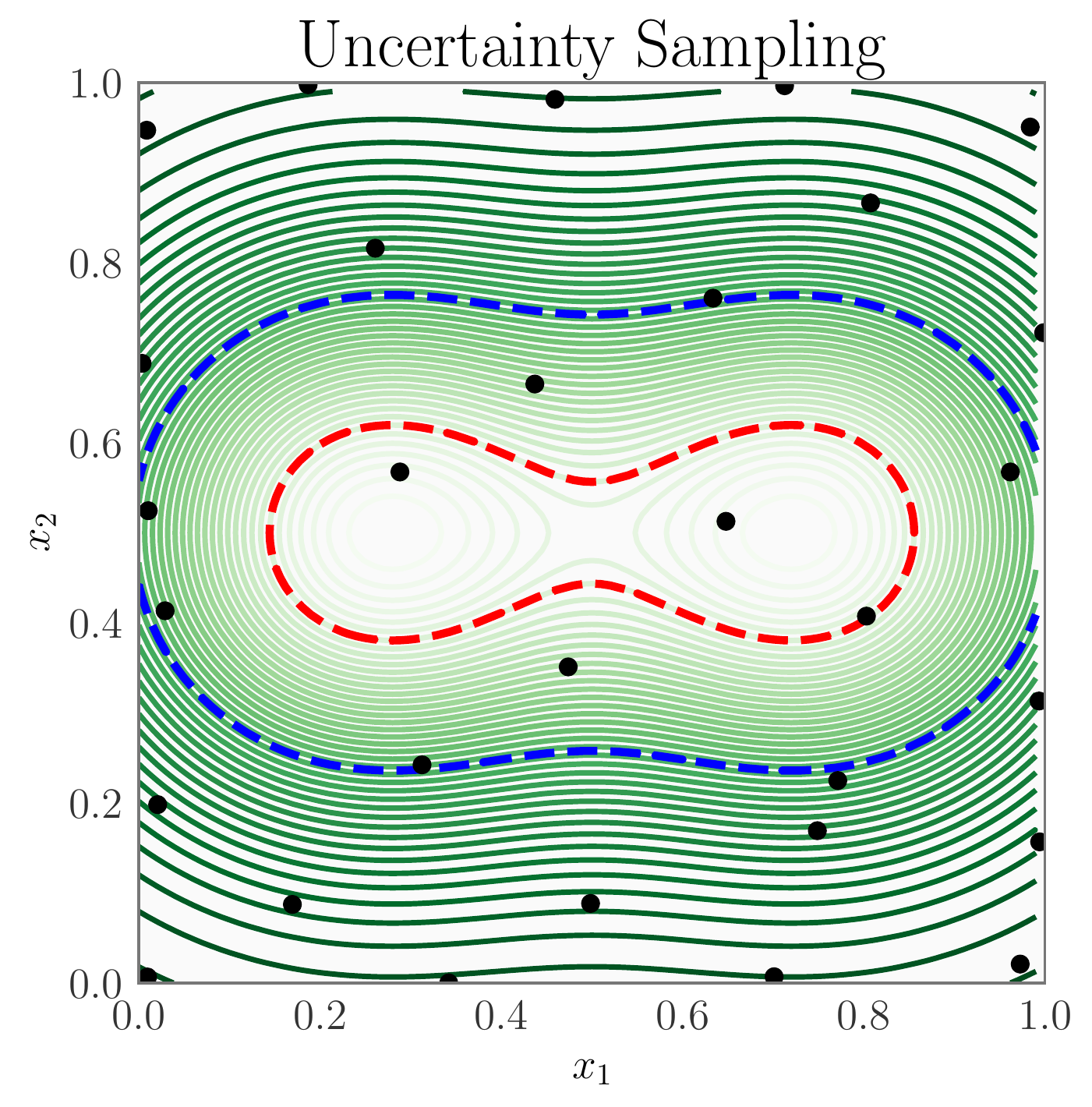}
\includegraphics[width=0.245\linewidth]{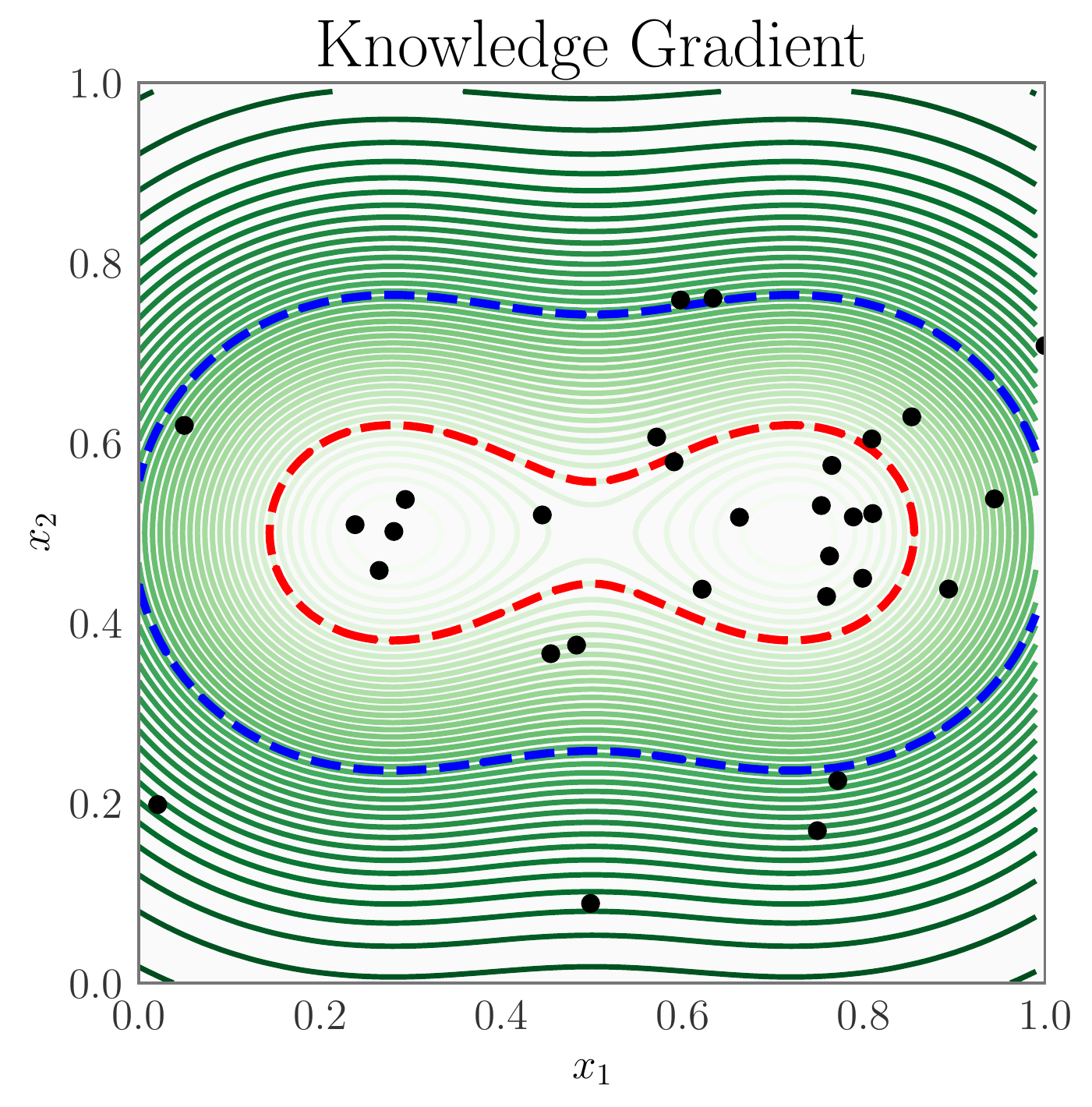}
\includegraphics[width=0.245\linewidth]{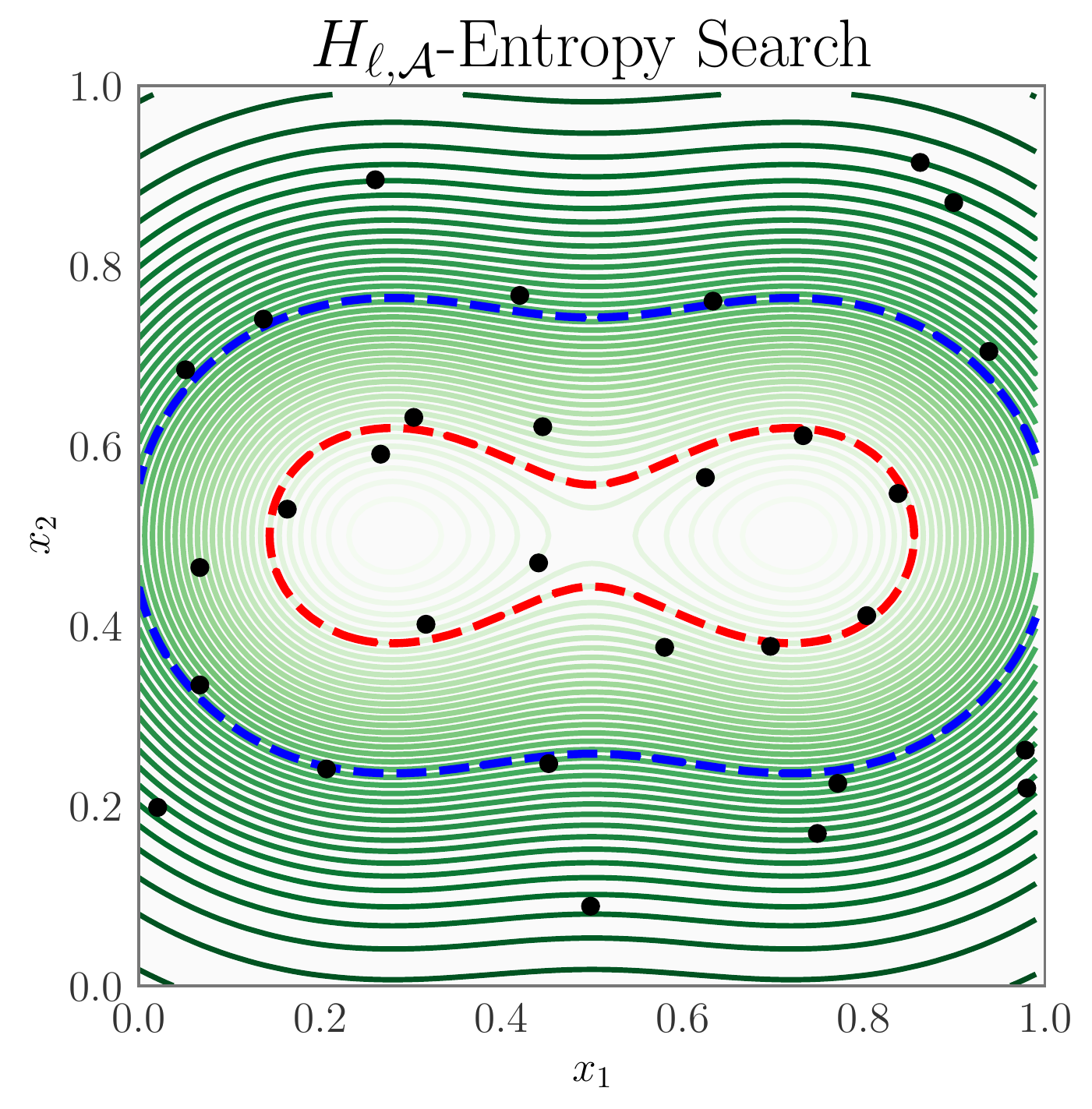}\\
\vspace{-2mm}
\caption{
\small
\textbf{Multi-level set estimation.}
\textit{Top row}: Plots of accuracy versus iteration for all methods,
where error bars represent one standard error.
\textit{Bottom row:} Visualization of methods on the \textit{Multihills}
function, showing the ground-truth level set boundaries (red and blue dashed lines)
and queries $\Dc_t$ taken (black dots).
}
\label{fig:levelset}
\vspace{-6mm}
\end{figure*}

\vspace{-3mm}
\paragraph{Multi-level Set Estimation}
In our second task, the goal is to carry out multi-level set estimation.
Here, we can assess each method using a more conventional metric:
we produce an estimate of the level for every $x \in \Xc_0$,
using the model's posterior mean (given the queries selected by a particular method),
and then record the accuracy of this estimate averaged across all level set thresholds.
Intuitively, a method will achieve a higher accuracy if it chooses queries that yield
a better estimate of the function near the threshold boundaries of the level sets.
In Figure~\ref{fig:levelset} (\textit{bottom row}), we visualize results for a two-level set task
on the \textit{Multihills} function, defined as a mixture density (details given in appendix).
We see that HES concentrates queries along both of the boundaries, which are drawn as blue and 
red dashed lines.
In the \textit{top row}, we compare the performance of all methods, showing the
accuracy vs. iteration.
Here, the \textit{Pennsylvania Night Light} 
function~\cite{nasa}
released by NASA (additional details in the appendix),
returns the relative level of light at a location in Pennsylvania, as queried by a satellite image.
The goal of this experiment is to determine the portion of land at which night light 
is above a specified threshold value.
In Appendix~\ref{sec:app-additionalexperiments}, we show additional experiments,
including a visualization of results on this function.
\begin{wrapfigure}{r}{0.5\textwidth}
\vspace{-1mm}
\centering
\includegraphics[width=0.52\linewidth]{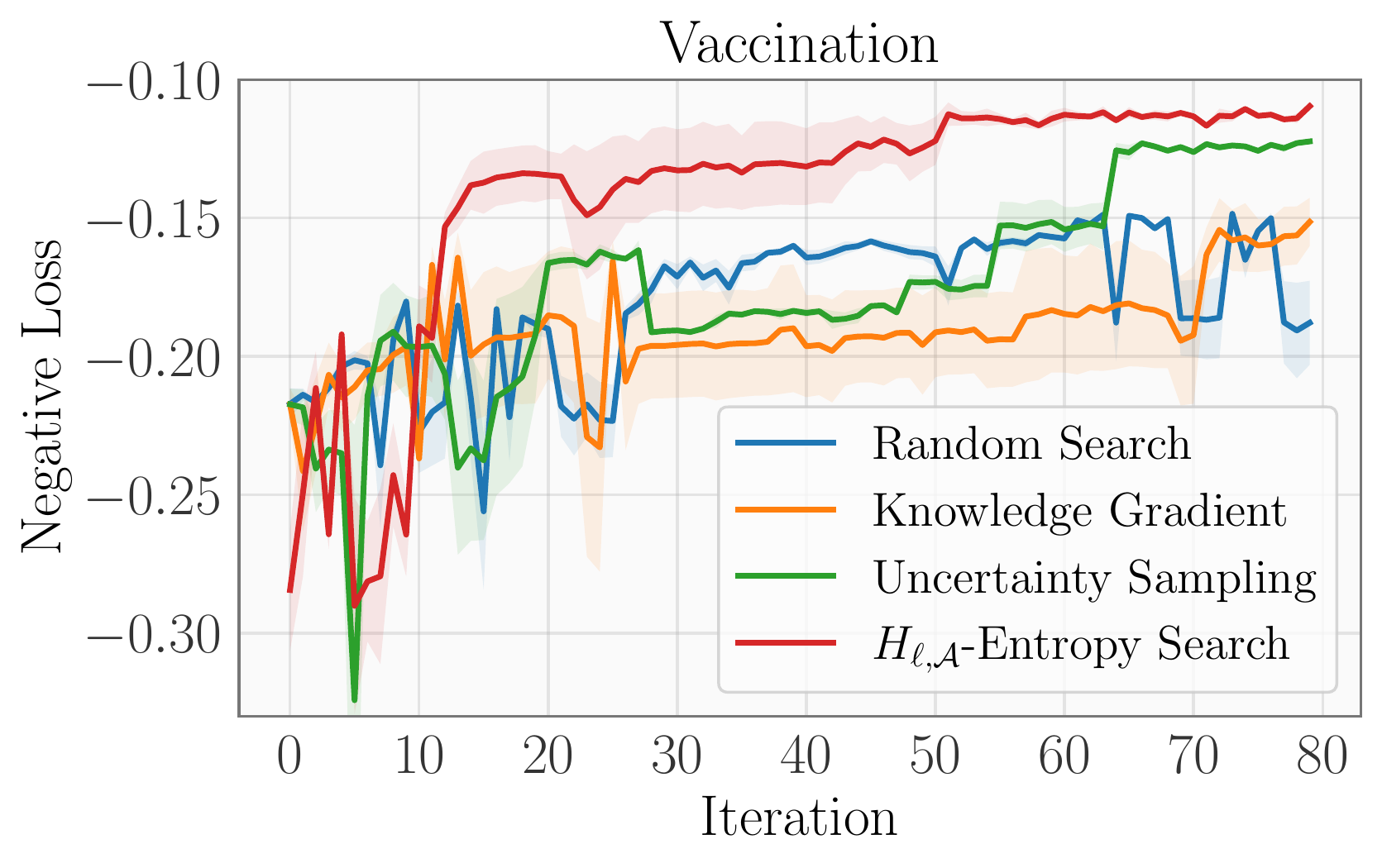}
\includegraphics[width=0.46\linewidth]{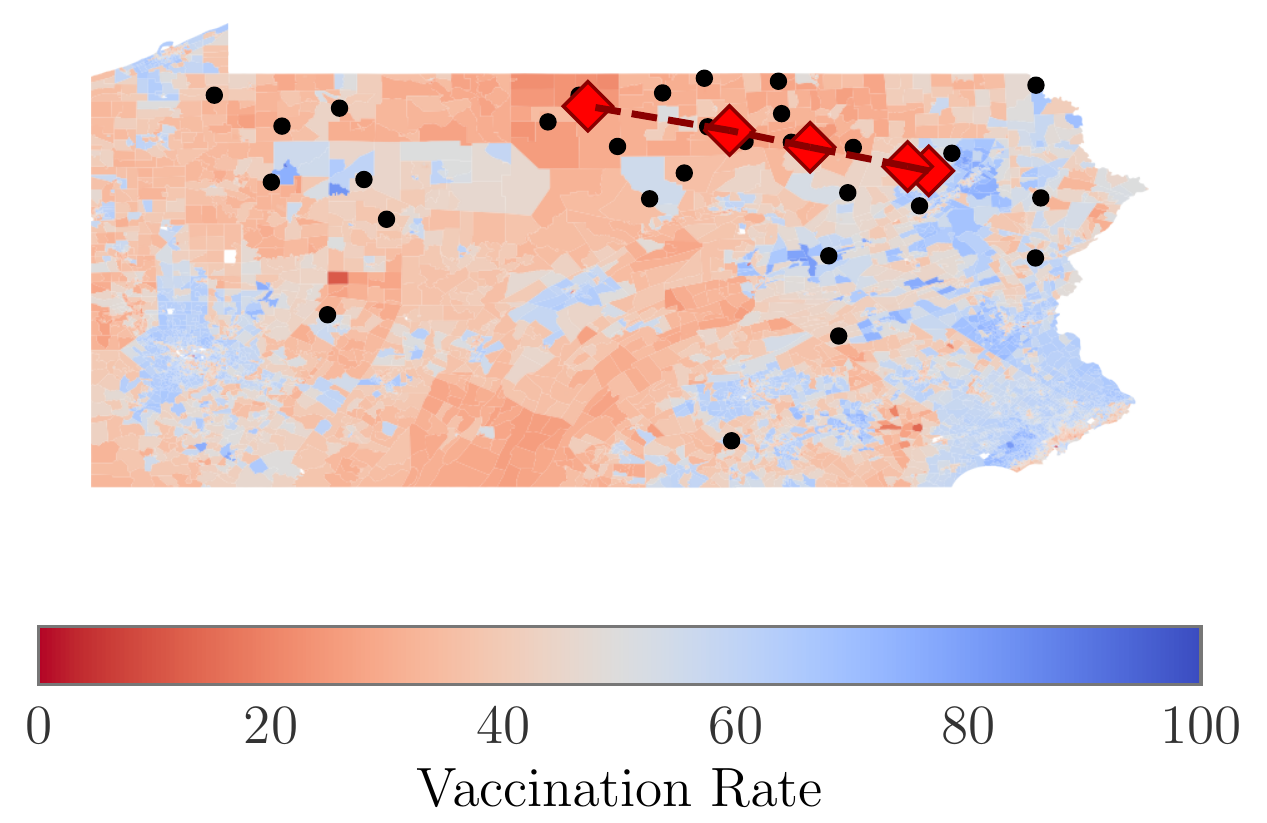}
\vspace{-3mm}
\caption{\small \textbf{Sequence search.}
\textit{Left:} Negative loss versus iteration, where error bars represent one standard error.
\textit{Right:} Visualization of the \textit{Vaccination} function, along with
the queries $\Dc_t$ taken by HES (black dots),
and the estimated sequence $(x_1^\circledast,\ldots,x_5^\circledast)$ (red diamonds),
such that $(f(x_1^\circledast),\ldots,f(x_5^\circledast))$
$=$ $(30\%, 40\%, 50\%, 60\%, 70\%)$.
}
\label{fig:mvs}
\vspace{-4mm}
\end{wrapfigure}

\vspace{-5mm}
\paragraph{Sequence Search}
In our third task, the goal is to find a sequence of elements whose value under the 
black-box function matches a set of pre-specified function values
$(y_1^\circledast, \ldots, y_m^\circledast)$.
To assess each method, at each iteration we record the negative loss
$-\ell(\truefunction, a^*)$ from \eqref{eq:mvsloss}---i.e. the
negative \textit{sequence search loss} of the Bayes action
$a^* = \arg\inf_{a \in \Ac} \expecf{p(f \mid \Dc_t)}{ \ell(f, a) }$---using
the set of queries $\Dc_t$ produced by the given method.
Intuitively, if a method makes a set of queries that yield a good estimate
of a sequence of $(x_1^\circledast, \ldots, x_m^\circledast)$ such that
$(\truefunction(x_1^\circledast), \ldots, \truefunction(x_m^\circledast))
\approx (y_1^\circledast, \ldots, y_m^\circledast)$,
it will score a high value on this metric.

In Figure~\ref{fig:mvs} (\textit{right}) we visualize results on
the  \textit{Vaccination} function (described above).
Here, our goal is to find a sequence of five 
\textit{(latitude, longitude)} coordinates with vaccination rates equal to
$(y_1^\circledast, \ldots, y_m^\circledast) = (30\%, 40\%, 50\%, 60\%, 70\%)$.
Estimates of locations that match these function values can be useful when making policy decisions
involving a vaccine response or allocation.
In this case, we see that HES concentrates queries along a route from the relatively highly
vaccinated region in the East to the relatively lowly vaccinated region in the North.
The \textit{left} plots in Figure~\ref{fig:mvs} provides a quantitive
comparison of methods on the \textit{Vaccination}
function (results on additional functions are shown in the appendix),
plotting  the negative loss vs. iteration.
These again show the benefits of query selection performed by HES
relative to the comparison strategies.
\section{Conclusion}
\label{sec:conclusion}

In this paper, we take a decision making perspective on information-based acquisition functions: 
after querying is complete, we assume that we must make some decision $a^*$ and then incur a loss
$\ell(\truefunction, a^*)$. Our goal is thus to make a sequence of queries that reduce the
uncertainty of the posterior distribution $p(f \mid \Dc_t)$ in a way to best help make this
decision with low loss.
Using $H_{\ell, \Ac}$-entropy \citep{DeGroot1962-ob, rao1984convexity}, we can define an EHIG acquisition
function which carries this out directly: it selects a point that is expected to maximally
reduce the posterior expected loss of the Bayes action $a^*$.
We incorporate this acquisition function into a procedure called \textsc{$H_{\ell, \Ac}$-Entropy Search},
and show, in many cases, that we can perform efficient gradient-based optimization
of this acquisition function.

There are multiple interesting avenues for future work.
First, we hope to develop acquisition optimization methods
for additional settings, such as for non-continuous action sets
$\Ac$ or design spaces $\Xc$ \citep{daulton2022bayesian}, and for functions with multidimensional outputs \citep{hernandez2016predictive, daulton2022multi}. One interesting avenue is hybrid optimization settings,
where we can only take gradient steps with respect to either the design or action variables.
Another potential direction is to incorporate cost-aware Bayesian optimization techniques into the EHIG framework \citep{lee2020cost, xiao2021amortized, astudillo2021multi}.
We also wish to study how the proposed EHIG framework could be
applied in practice to solve various problems in the sciences, including experimental physics \citep{duris2020bayesian, miskovich2022bayesian, char2019offline}, drug discovery \citep{stanton2022accelerating, kassraie2022graph, griffiths2022gauche}, and materials design \citep{lookman2019active, tran2020methods}.
Finally, we wish to study in further detail how the EHIG acquisition function could be implemented for Bayesian decision making with other probabilistic models beyond Gaussian processes \citep{snoek2015scalable, chung2021beyond, chung2021uncertainty, neiswanger2019probo}.
\subsection*{Acknowledgments}

We thank the anonymous reviewers, members of the Stanford SAIL community, and members of the CMU Auton Lab for helpful feedback on this paper.
This work was supported by NSF (\#1651565), AFOSR (FA95501910024), ARO (W911NF-21-1-0125), CZ Biohub, and Sloan Fellowship.
\bibliography{refs}
\bibliographystyle{plain}
\section*{Checklist}


\begin{enumerate}

\item For all authors...
\begin{enumerate}
  \item Do the main claims made in the abstract and introduction accurately reflect the paper's contributions and scope?
    \answerYes{}
  \item Did you describe the limitations of your work?
    \answerYes{}
  \item Did you discuss any potential negative societal impacts of your work?
    \answerNA{}
  \item Have you read the ethics review guidelines and ensured that your paper conforms to them?
    \answerYes{}
\end{enumerate}

\item If you are including theoretical results...
\begin{enumerate}
  \item Did you state the full set of assumptions of all theoretical results?
    \answerYes{} See Section~\ref{sec:existingacqfunctions} and appendix Section~\ref{sec:app-proofs}.
        \item Did you include complete proofs of all theoretical results?
    \answerYes{} See appendix Section~\ref{sec:app-proofs}.
\end{enumerate}

\item If you ran experiments...
\begin{enumerate}
  \item Did you include the code, data, and instructions needed to reproduce the main experimental results (either in the supplemental material or as a URL)?
    \answerYes{} All code and instructions are included in supplementary material.
  \item Did you specify all the training details (e.g., data splits, hyperparameters, how they were chosen)?
    \answerYes{} All training details are specified in the paper and included code.
        \item Did you report error bars (e.g., with respect to the random seed after running experiments multiple times)?
    \answerYes{} See Section~\ref{sec:experiments}.
        \item Did you include the total amount of compute and the type of resources used (e.g., type of GPUs, internal cluster, or cloud provider)?
    \answerYes{} See appendix~\ref{sec:app-additionalexperiments}. 
\end{enumerate}

\item If you are using existing assets (e.g., code, data, models) or curating/releasing new assets...
\begin{enumerate}
  \item If your work uses existing assets, did you cite the creators?
    \answerYes{} All existing assets were cited.
  \item Did you mention the license of the assets?
    \answerYes{} Licences of all assets are properly attributed.
  \item Did you include any new assets either in the supplemental material or as a URL?
    \answerYes{} New assets are included in the supplementary material.
  \item Did you discuss whether and how consent was obtained from people whose data you're using/curating?
    \answerNA{}
  \item Did you discuss whether the data you are using/curating contains personally identifiable information or offensive content?
    \answerNA{}
\end{enumerate}

\item If you used crowdsourcing or conducted research with human subjects...
\begin{enumerate}
  \item Did you include the full text of instructions given to participants and screenshots, if applicable?
    \answerNA{}
  \item Did you describe any potential participant risks, with links to Institutional Review Board (IRB) approvals, if applicable?
    \answerNA{}
  \item Did you include the estimated hourly wage paid to participants and the total amount spent on participant compensation?
    \answerNA{}
\end{enumerate}

\end{enumerate}

\newpage
\appendix

\section{Proofs}
\label{sec:app-proofs}

Here we prove the propositions stated in Section~\ref{sec:existingacqfunctions}.

\subsection{Entropy Search}
\label{sec:app-proof-es}

\textbf{Proposition~\ref{prop:equiv_es}.}
If we choose $\Ac = \Pc(\Theta)$ and $\ell(f, q) = - \log q( \theta_f )$, then the $\mathrm{EHIG}$ is equivalent to the entropy search acquisition function,
i.e. $\mathrm{EHIG}_t(x; \ell, \Ac) = \mathrm{ES}_t(x)$.

\begin{proof}[Proof of Proposition~\ref{prop:equiv_es}]
We first prove that under our definition of loss $\ell$, the $H_{\ell, \Ac}$-entropy $H[f \mid \Dc_t]$ is equivalent to the Shannon entropy of the posterior distribution over $\theta_f$ (where $\theta_f$ denotes a property of $f$ that we would like to infer---as an example, $\theta_f$ could be equal to the global maximizer $x^*$ of $f$).

Note that the $H_{\ell, \Ac}$-entropy is the expected loss of the Bayes action
\begin{align*}
q^* = {\arg\inf}_{q \in \Pc(\Xc)} \expecf{p(f|\Dc_t)}{-\log q(\theta_f)}.
\end{align*}
We want to show that $q^*$ defined above is equal to $p(\theta_f \mid \Dc_t)$. To do so, note that
\begin{align}
q^* &= {\arg\inf}_{q \in \Pc(\Xc)} \expecf{p(f|\Dc_t)}{-\log q( \theta_f |\Dc_t)} \\
&= {\arg\inf}_{q \in \Pc(\Xc)} \expecf{p( \theta_f |\Dc_t)}{-\log q( \theta_f |\Dc_t)}\\
&= p( \theta_f |\Dc_t),
\end{align}
where the first equality holds since
\begin{align} 
    E_X[f(g(X))] = E_Z[f(Z)], \text{when } Z = g(X),
\end{align}
and the second equality holds since we can view
$\expecf{p( \theta_f |\Dc_t)}{-\log q( \theta_f |\Dc_t)}$
as a cross entropy,
which is minimized when $q( \theta_f |\Dc_t) = p( \theta_f |\Dc_t)$.
Therefore, under this loss and action set, using the definition of the EHIG we can write
\begin{align} 
\mathrm{EHIG}_t(x; \ell, \Ac)
= H\left[ p(\theta_f \mid \Dc_t) \right]
- \Eb_{p(y_x \mid \Dc_t)} \left[ H\left[ p(\theta_f \mid \Dc_t \cup \lbrace x, y_x \rbrace) \right] \right]
= \mathrm{ES}_t(x).
\end{align}

\end{proof}

\subsection{Knowledge Gradient}
\label{sec:app-proof-kg}

\textbf{Proposition 2.}
If we choose $\Ac = \Xc$ and $\ell(f, x) = -f(x)$, then the $\mathrm{EHIG}$ is equivalent to the knowledge gradient acquisition function,
i.e. $\mathrm{EHIG}_t(x; \ell, \Ac) = \mathrm{KG}_t(x)$.

\begin{proof}[Proof of Proposition 2.]
The proof follows directly from the definition of $H_{\ell, \Ac}$-entropy and the EHIG, namely
\begin{align}
    \text{EHIG}_t(x)
    &= \inf_{a \in \Ac} \expecf{p(f \mid \Dc_t)}{ \ell(f, a) }
    -  \expecf{p(y_x | \Dc_t)}{
    \inf_{a \in \Ac} \expecf{p(f \mid \Dc_t \cup \{(x, y_x)\})}{ \ell(f, a) }
    }\\
    &= \inf_{x' \in \Xc} \expecf{p(f \mid \Dc_t)}{ -f(x') }
    -  \expecf{p(y_x | \Dc_t)}{
    \inf_{x' \in \Xc} \expecf{p(f \mid \Dc_t \cup \{(x, y_x)\})}{ -f(x') }
    }\\
    &= - \sup_{x' \in \Xc} \expecf{p(f \mid \Dc_t)}{ f(x') }
    +  \expecf{p(y_x | \Dc_t)}{
    \sup_{x' \in \Xc} \expecf{p(f \mid \Dc_t \cup \{(x, y_x)\})}{ f(x') }
    }\\
    &= \expecf{p(y_x | \Dc_t)}{ \mu_{t+1}^*(x, y_x) } - \mu_t^*\\
    &= \text{KG}_t(x) \hspace{4mm}
\end{align}

\end{proof}

\subsection{Expected Improvement}
\label{sec:app-proof-ei}

\textbf{Proposition 3.}
If we choose $\Ac_t = \{x_i\}_{i=1}^{t-1}$, where $x_i \in \Dc_t$,
and $\ell(f, x_i) = -f(x_i)$, then the $\mathrm{EHIG}$ is equal 
to the expected improvement acquisition function, i.e.
$\mathrm{EHIG}_t(x; \ell, \Ac) = \mathrm{EI}_t(x)$.

\begin{proof}[Proof of Proposition 3]
The first term of $\mathrm{EHIG}_t$ in \eqref{eq:ehig} is equal to:
\begin{align} 
H_{\ell, \Ac_t}[f \mid \Dc_t] =\inf_{a \in \Ac_t} \expecf{p(f \mid \Dc_t)}{ \ell(f, a) } = -\max_{i \leq t-1} \hat{f}(x_i) := -f^*_t
\end{align}
where $\hat{f}(x_i)$ is the posterior expected value of $f$ at $x_i$.

The second term in \eqref{eq:ehig} is:
\begin{align}
    &\expecf{p(y_x | \Dc_t)}{
        H_{\ell, \Ac_{t+1}} \left[f \mid \Dc_t \cup \{(x, y_x)\} \right]
    }\\
    =&\expecf{p(y_x | \Dc_t)}{
        \expecf{p(f \mid \Dc_t \cup \{(x, y_x)\})}{\inf_{a \in A_{t+1}} \ell(f, a) }
    }\\
    =& \expecf{p(y_x | \Dc_t)}{
     \expecf{p(f \mid \Dc_t \cup \{(x, y_x)\})}{-\max(f^*_t, f(x))}
    }\\
    =& \expecf{p(y_x | \Dc_t)}{
     -\max(f^*_t, y_x)
    }
\end{align}
Putting it together, the $\mathrm{EHIG}_t$ acquisition function in \eqref{eq:ehig} will reduce to:
\begin{align}
    \mathrm{EHIG}_t(x; \ell, \Ac)
    &= - f^*_t - \expecf{p(y_x | \Dc_t)}{
     -\max(f^*_t, y_x)
    }\\
    &=  \Eb_{p(y_x \mid \Dc_t)}[\max(0, y_x- f^*_t) ]\\
    &= \mathrm{EI}_t(x).
\end{align}

\end{proof}

\subsection{Probability of Improvement}
\label{sec:app-proof-pi}

We additionally include a result below showing that the probability of improvement (PI) acquisition function can similarly be viewed as a special case of the proposed EHIG family.

\textbf{Proposition 4.}
For some constant $\tau$, the acquisition function of PI is defined as $\mathrm{PI}_\tau(x; \mathcal{D}_t) = \mathbb{E}_{p(f|\mathcal{D}_t)} [\mathbb{I}(f(x) - \tau > 0)]$, where $\mathbb{I}(\cdot)$ is the indicator function, and typically $\tau$ is taken to be equal to $f_t^* = \max_{i \leq t-1} \hat{f}(x_i)$ for $x_i \in \mathcal{D}_t$.
If we choose $\Ac_t = \{x_{t-1}\}$, where $x_{t - 1} \in \Dc_t$,
and $\ell_\tau(f, x) = -\mathbb{I}(f(x) - \tau > 0)$,
then maximizing $\mathrm{EHIG}$ is equivalent to maximizing the probability of improvement acquisition function, i.e.
$\argmax_{x \in \mathcal{X}} \mathrm{EHIG}_t(x; \ell_\tau, \Ac) = \argmax_{x \in \mathcal{X}} \mathrm{PI}_\tau(x)$.

\begin{proof}[Proof of Proposition 4]
The first term of $\mathrm{EHIG}_t$ in \eqref{eq:ehig} is equal to:
\begin{align} 
H_{\ell, \Ac_t}[f \mid \Dc_t] =\inf_{a \in \Ac_t} \expecf{p(f \mid \Dc_t)}{ \ell(f, a) } = -\mathbb{I}(\hat{f}(x_{t-1}) - \tau > 0)
\end{align}
where $\hat{f}(x_{t-1})$ is the posterior expected value of $f$ at $x_{t-1}$. More importantly, $H_{\ell, \Ac_t}[f \mid \Dc_t]$ is a constant with respect to $x$ that we are optimizing.

The second term in \eqref{eq:ehig} is:
\begin{align}
    &\expecf{p(y_x | \Dc_t)}{
        H_{\ell, \Ac_{t+1}} \left[f \mid \Dc_t \cup \{(x, y_x)\} \right]
    }\\
    =&\expecf{p(y_x | \Dc_t)}{
        \inf_{a \in \{x\}}
        \expecf{p(f \mid \Dc_t \cup \{(x, y_x)\})}{\ell(f, a) }
    }\\
    =& \expecf{p(y_x | \Dc_t)}{
     \expecf{p(f \mid \Dc_t \cup \{(x, y_x)\})}{-\mathbb{I}(f(x) - \tau > 0)}
    }\\
    =& -\expecf{p(y_x | \Dc_t)}{
     \mathbb{I}(y_x - \tau > 0)
    }
\end{align}
Putting it together, the $\mathrm{EHIG}_t$ acquisition function in \eqref{eq:ehig} will reduce to:
\begin{align}
    \mathrm{EHIG}_t(x; \ell_\tau, \Ac)
    &= -\mathbb{I}(\hat{f}(x_{t-1}) - \tau > 0) + \expecf{p(y_x | \Dc_t)}{
     \mathbb{I}(y_x - \tau > 0)
    }\\
    &=  \expecf{p(y_x | \Dc_t)}{
     \mathbb{I}(y_x - \tau > 0)
    } + \mathrm{constant}\\
    &= \mathrm{PI}_\tau(x) + \mathrm{constant}.
\end{align}
Thus maximizing $\mathrm{EHIG}$ is equivalent to maximizing the probability of improvement acquisition function.

\end{proof}

\section{Additional Experimental Details and Results}
\label{sec:app-additionalexperiments}
\label{app:sec:dataset}

\paragraph{Details on the \textit{Alpine-$d$} function.}
The multimodal \textit{Alpine-$d$} function is defined as
$\textit{Alpine-}d(x) = \sum_{i=1}^d |x_i \sin(x_i) + 0.1 x_i |$, for $x \in \Rbb^d$.

\paragraph{Details on the \textit{Vaccination} function.}
The vaccination function is obtained by training a Multi-Layer Perceptron (MLP) network based on the data from \cite{yuan2021mobility}, which uses county-level vaccination data provided by the CDC, and uses small area estimation\footnote{\url{https://en.wikipedia.org/wiki/Small_area_estimation}} to interpolate the vaccination rate of every location. We restrict the optimization domain to be a rectangle focusing on the state of Pennsylvania.

\paragraph{Details on the \textit{Multihills} function.}
The \textit{Multihills} function is defined as a mixture density as follows.
$\textit{Multihills}(x) = \sum_{j=1}^J w_j \Nc(x \mid \mu_j, C_j)$, for $x \in \Rbb^d$,
where $\Nc$ denotes a multivariate normal density, $\{\mu_j\}$ are a set of $J$ means, 
$\{C_j\}$ are a set of J covarance matrices, and $\{w_j\}$ are a set of J weights.

\paragraph{Details on the \textit{Pennsylvania Night Light} function.}
We consider the  2012  gray  scale global nightlight raster with resolution 0.1 degree per pixel. The data is downloaded from NASA Earth Observatory\footnote{\url{https://earthobservatory.nasa.gov/features/NightLights}}. We restrict the optimization domain to be a rectangle focusing on the state of Pennsylvania and normalize all
raster data before use.
Each location query gives a value proportional to the average amount of night light at that location.

\paragraph{Computational Cost.}
While using the $\textrm{EHIG}_t(x; \ell, \Ac)$ acquisition function in Bayesian optimization
(Algorithm~\ref{alg:hes}) is more expensive than simpler methods (e.g. expected improvement (EI)),
in many cases it has a comparable computational cost to methods such as knowledge gradient (KG)
or entropy search (ES) methods, when applied to the same task---in fact, our implementation has
a similar structure as one-shot knowledge gradient acquisition optimization methods.

The following timing results compare the average cost (\textit{mean wall clock time in seconds})
of acquisition optimization for a set of comparison methods, including EI as an additional method, on the \textit{Alpine-$2$} function from the first experiment in our paper:
\textbf{\textit{EHIG: 6.9s, KG: 6.6s, EI: 0.5s, US: 0.3s}}.

\newpage

\subsection{Additional Experiment Results and Visualizations.}

We show further experiment results for multi-level set estimation and sequence search (Figure 5), visualizations for multi-level set estimation (Figure 6), and an additional comparisons of classic BO acquisition functions on the initial top-$k$ optimization experiments (Figure 7).

\begin{figure*}[h!]
\centering
\includegraphics[width=0.3\linewidth]{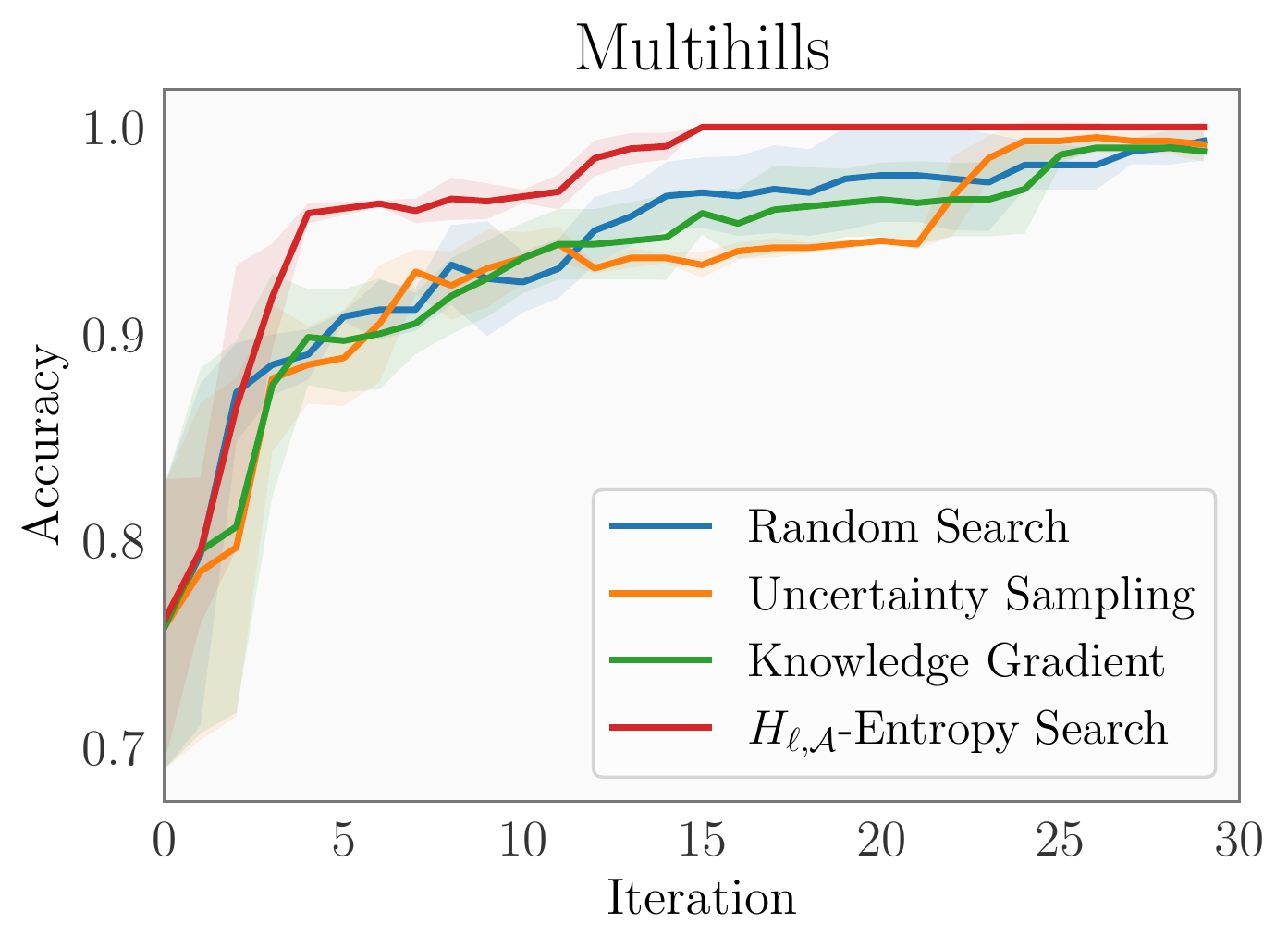}
\includegraphics[width=0.3\linewidth]{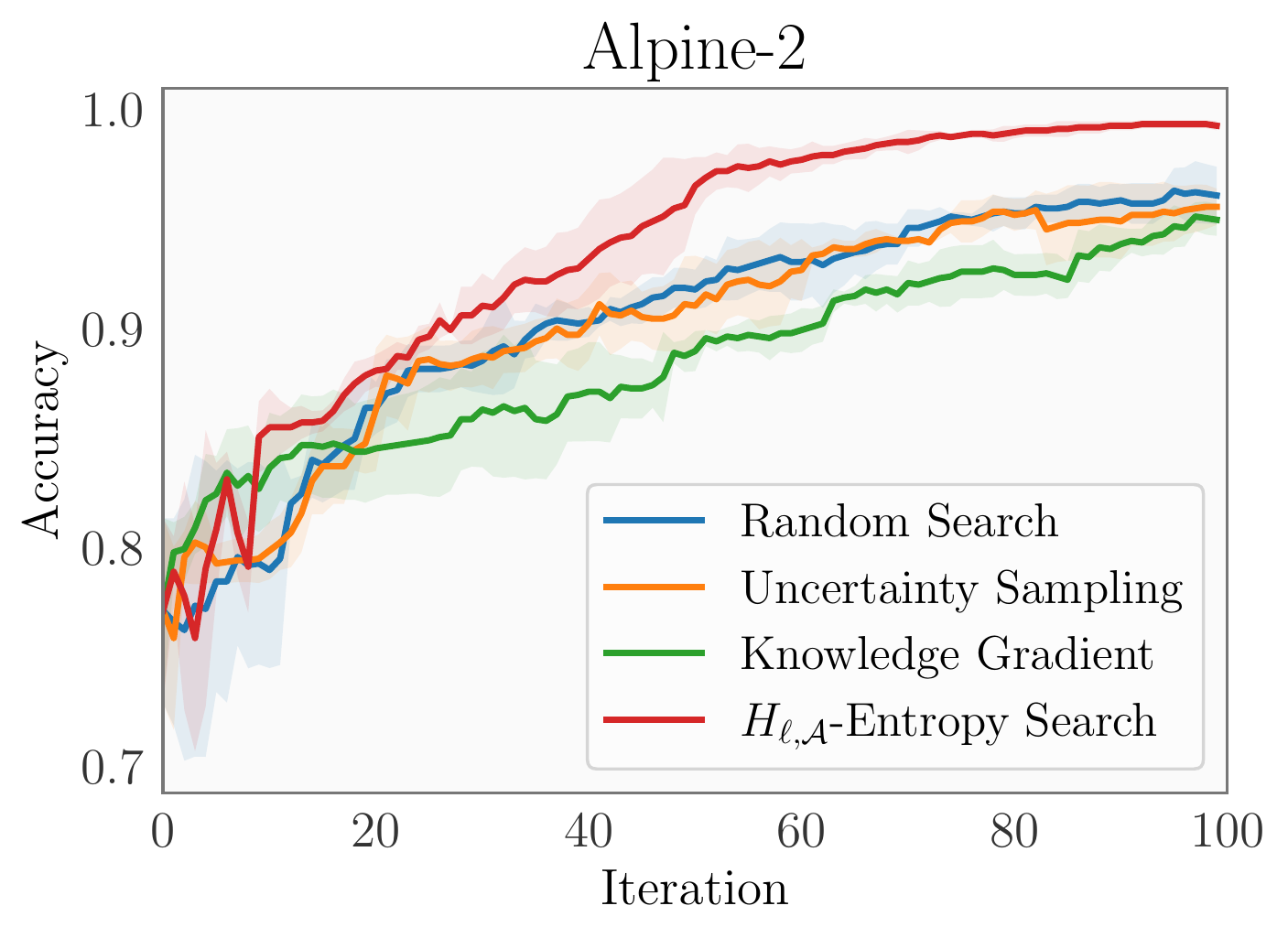}
\includegraphics[width=0.3\linewidth]{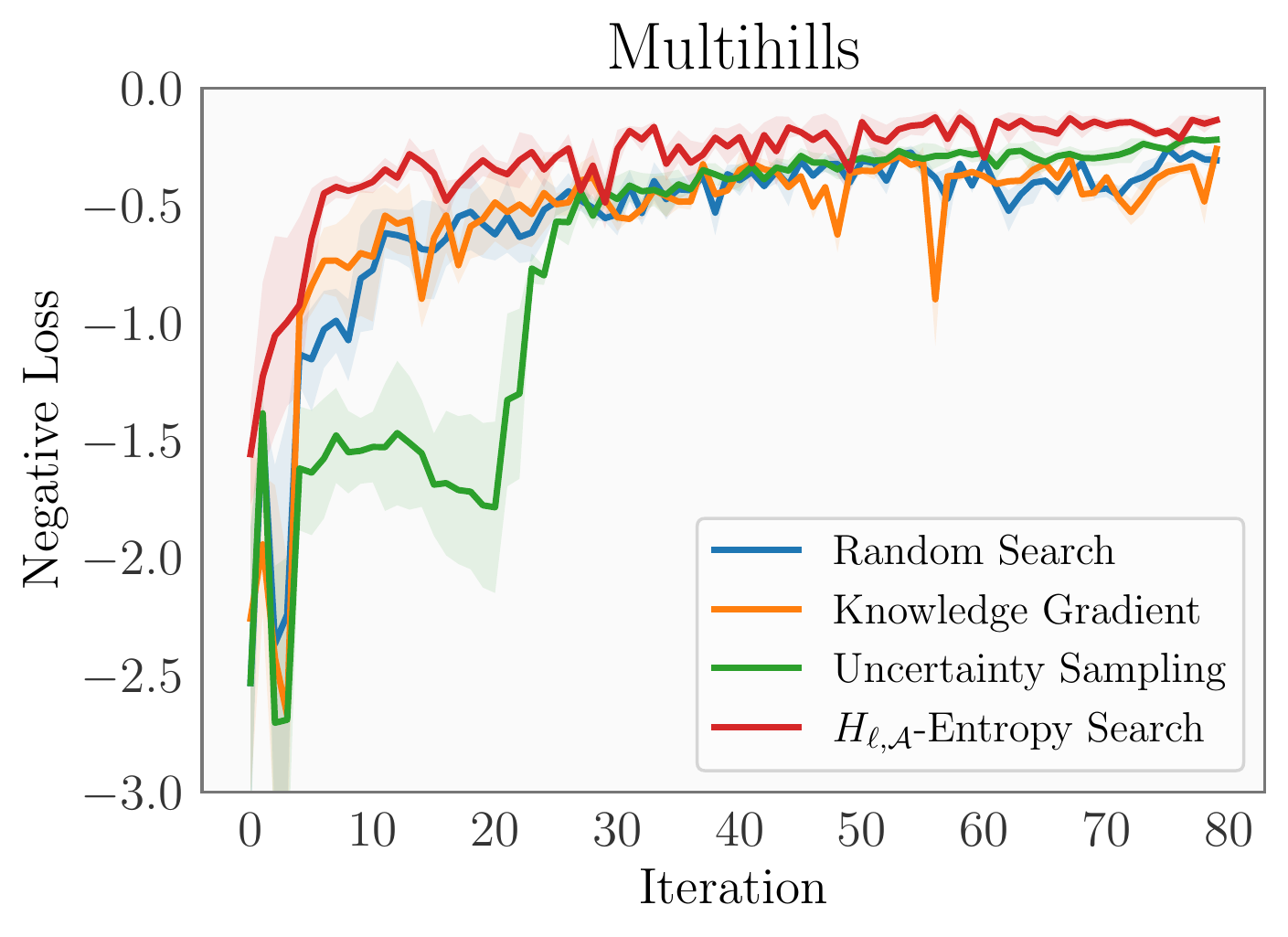}
\caption{
\small
\textbf{Multi-level set estimation and sequence search.}
\textit{Left and center}: Plots of accuracy versus iteration for the task of multi-level set estimation
(Equation~(\ref{eq:mlseloss}), $m=1$), where error bars represent one standard error.
\textit{Right}: Plot of negative loss versus iteration for the task of sequence search
(Equation~(\ref{eq:mvsloss})), where error bars represent one standard error.
}
\label{fig:app-levelset-curves}
\end{figure*}

\begin{figure}[h!]
\centering
\includegraphics[width=0.25\linewidth]{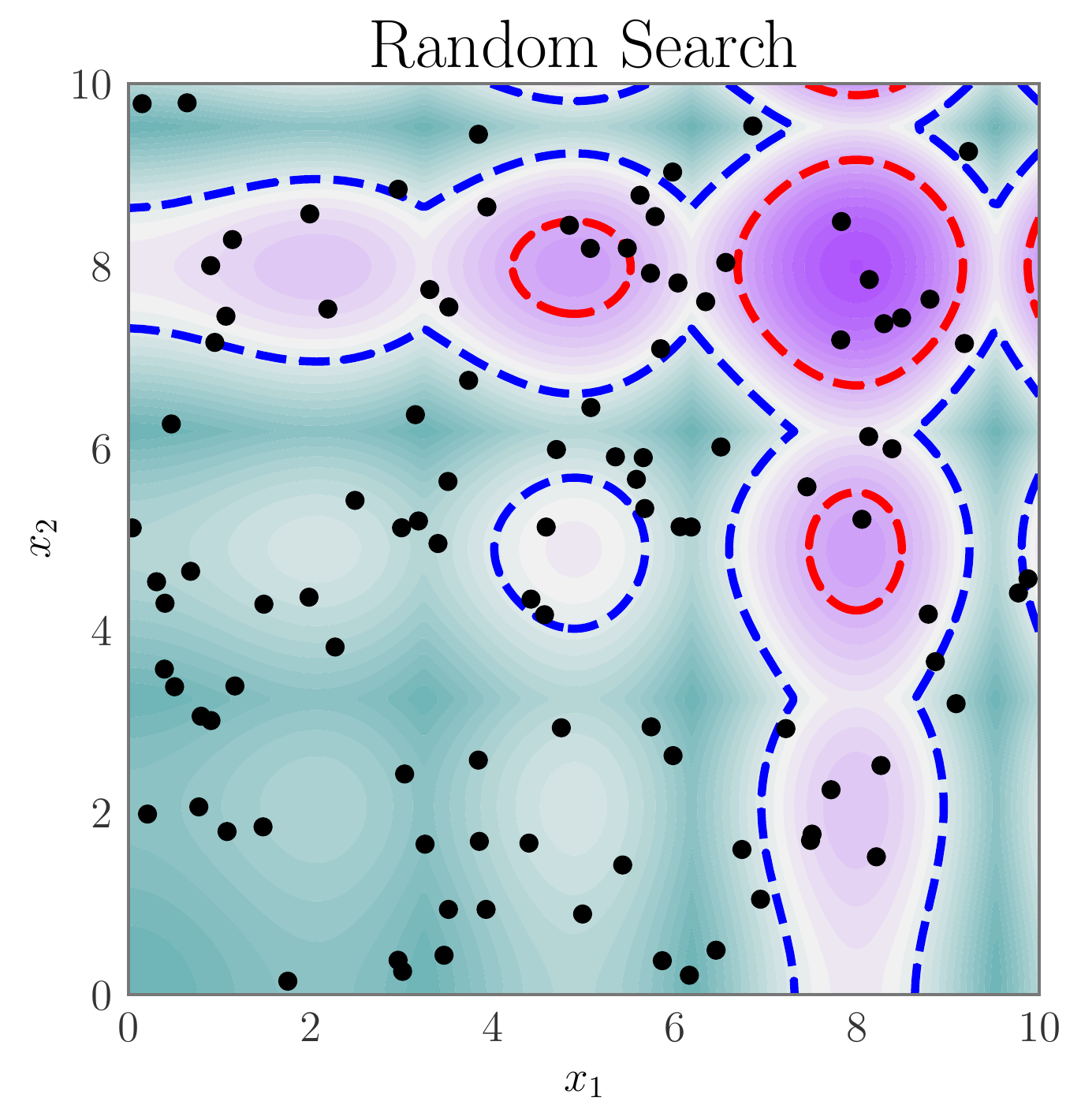}\hfill
\includegraphics[width=0.25\linewidth]{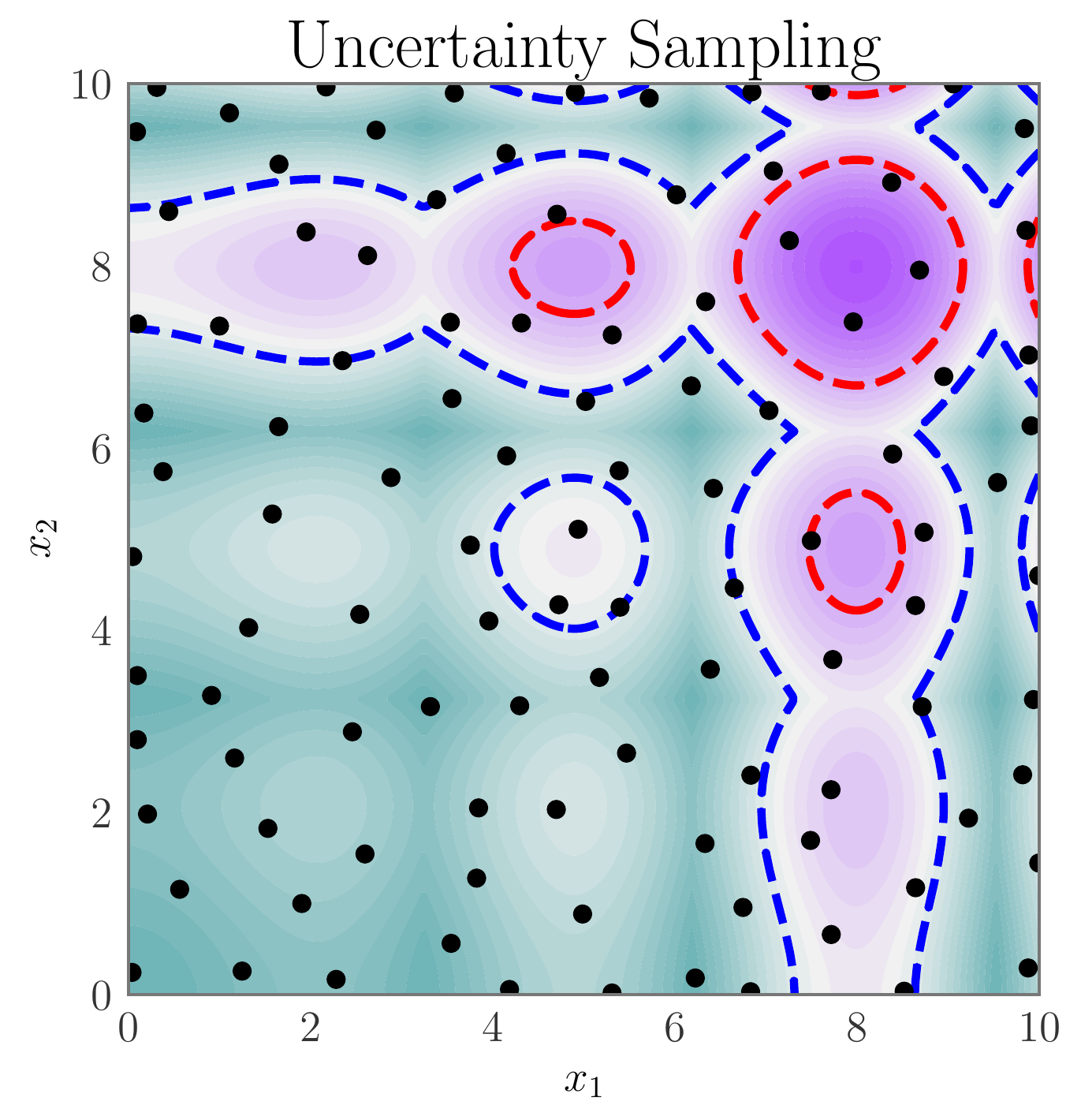}\hfill
\includegraphics[width=0.25\linewidth]{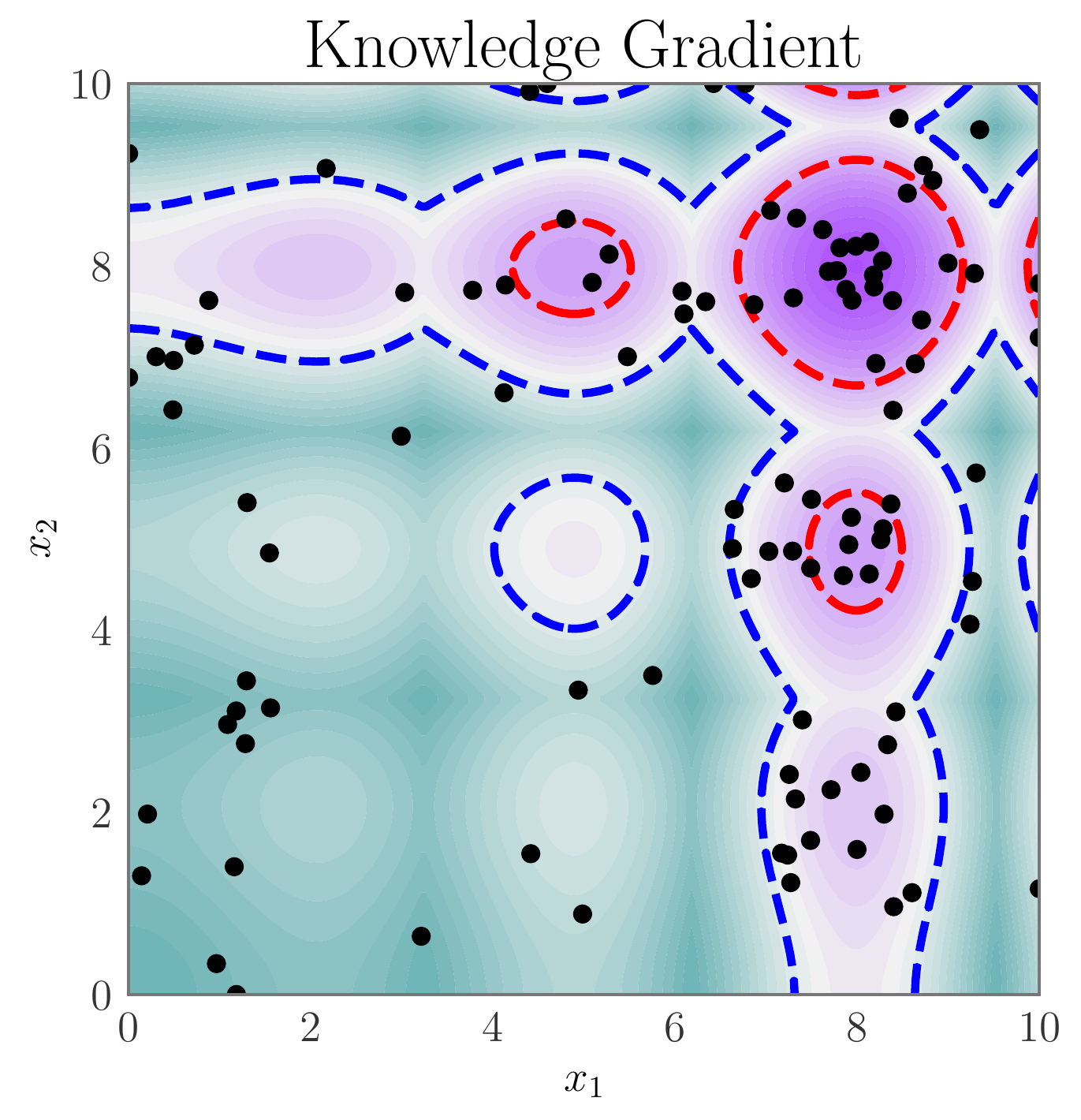}\hfill
\includegraphics[width=0.25\linewidth]{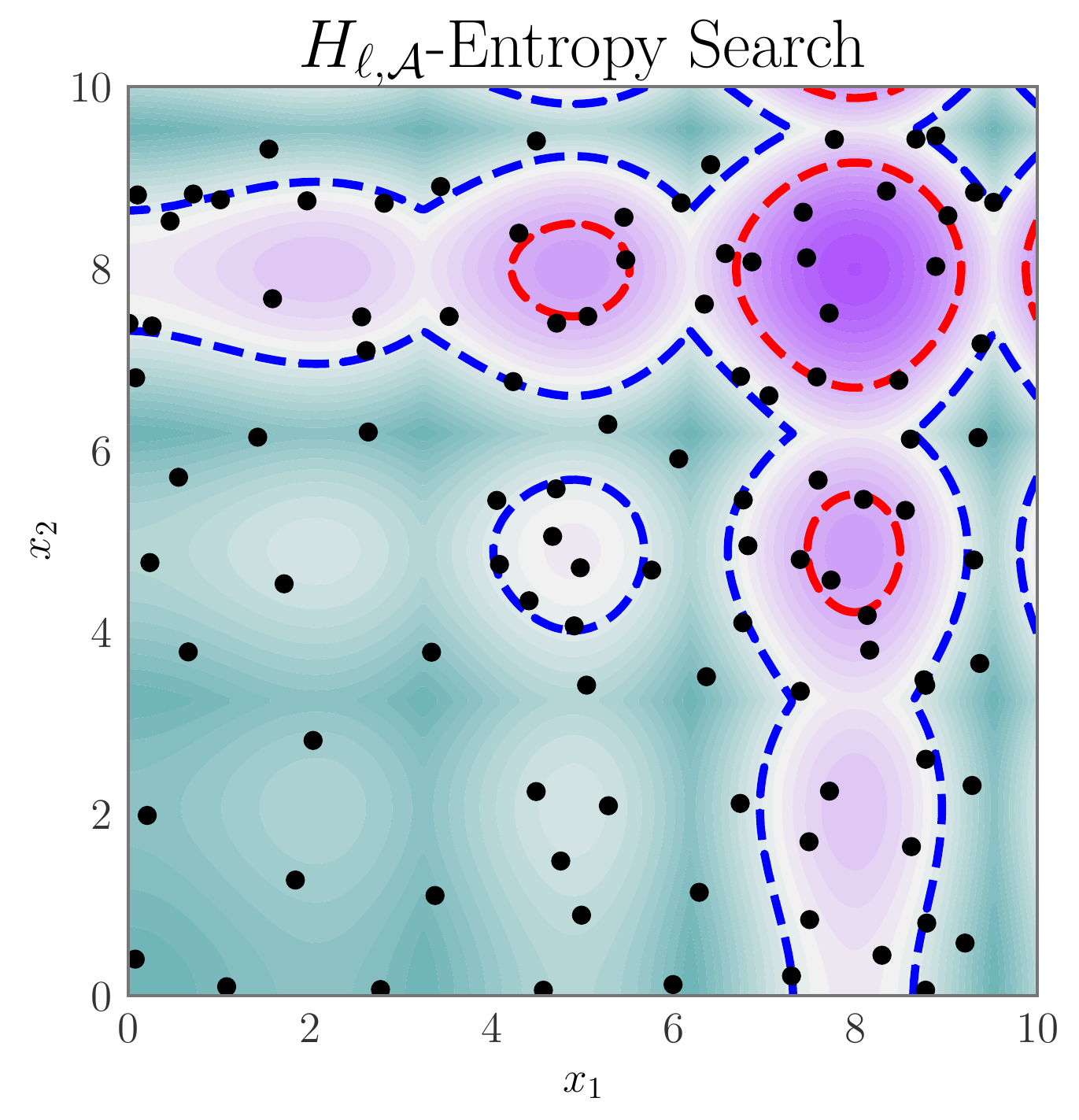}\\
\includegraphics[width=0.25\linewidth]{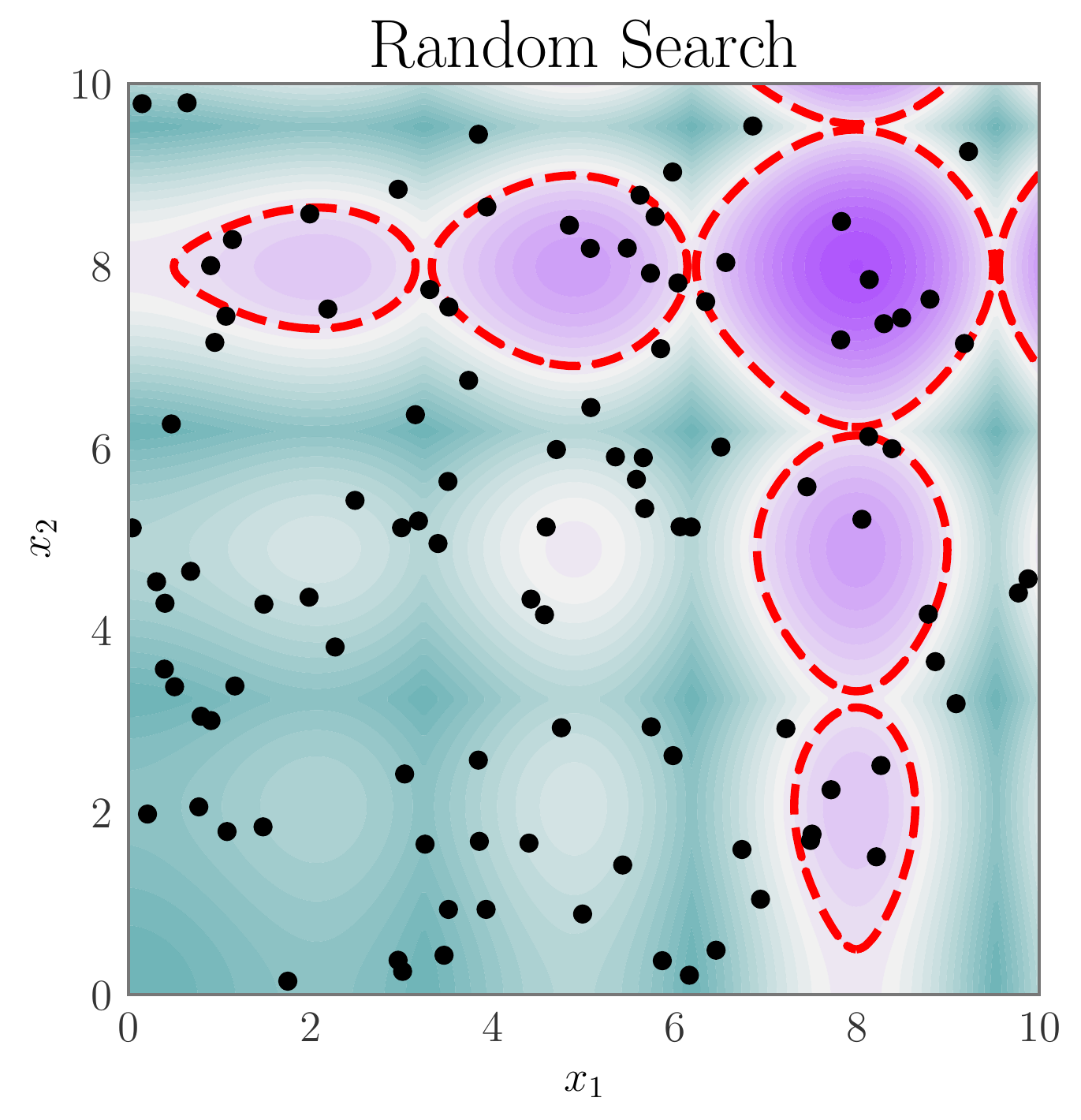}\hfill
\includegraphics[width=0.25\linewidth]{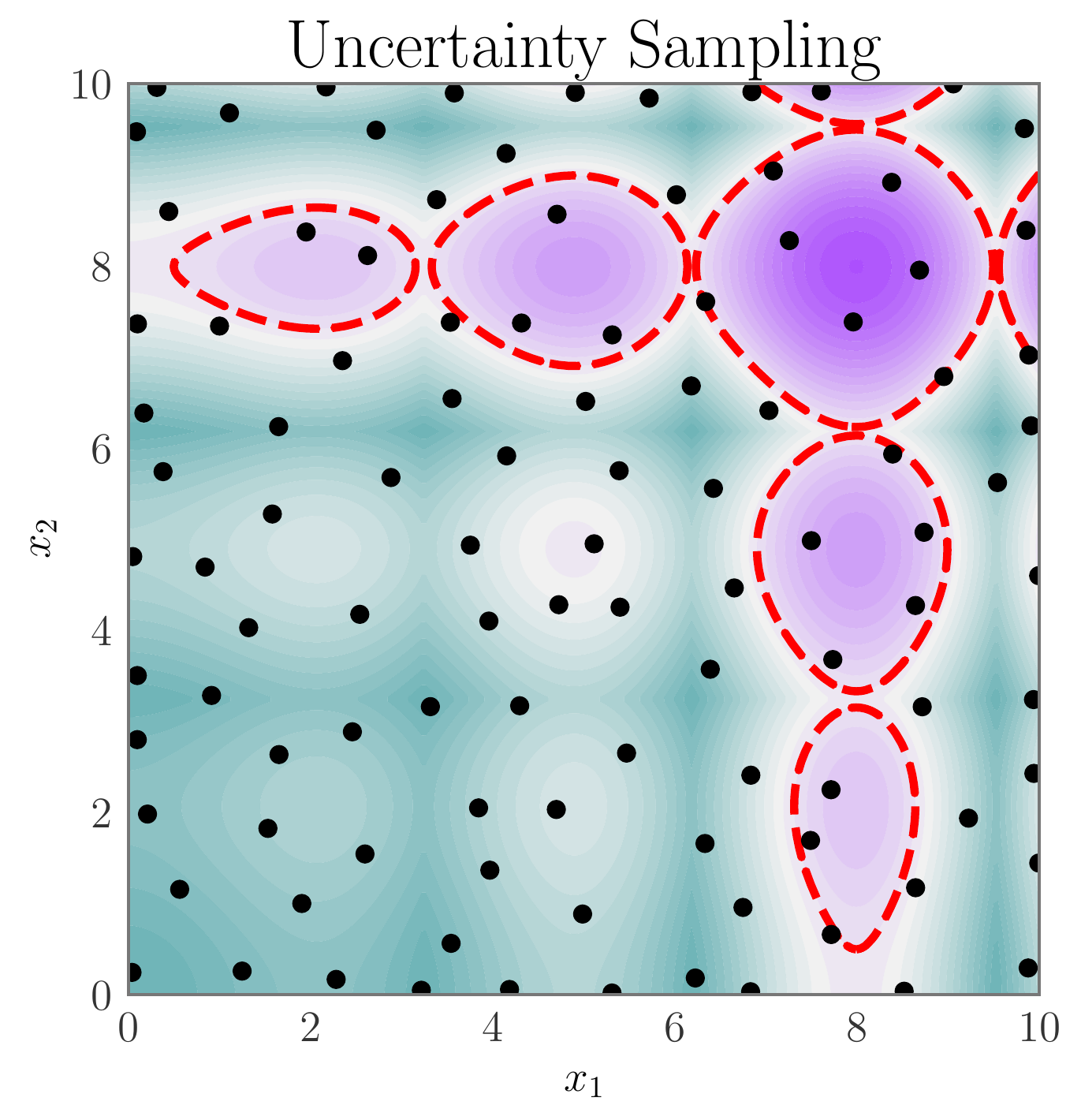}\hfill
\includegraphics[width=0.25\linewidth]{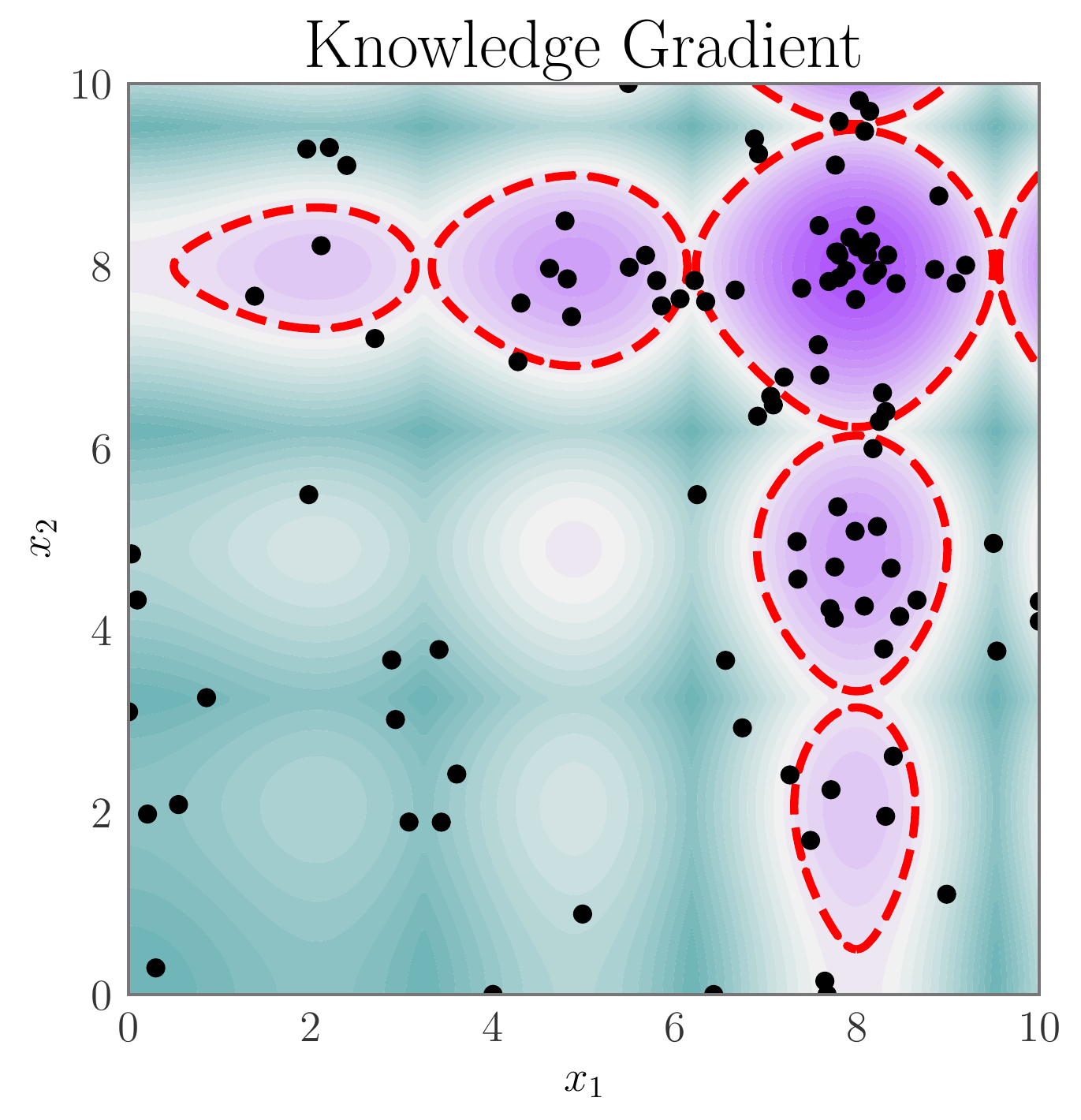}\hfill
\includegraphics[width=0.25\linewidth]{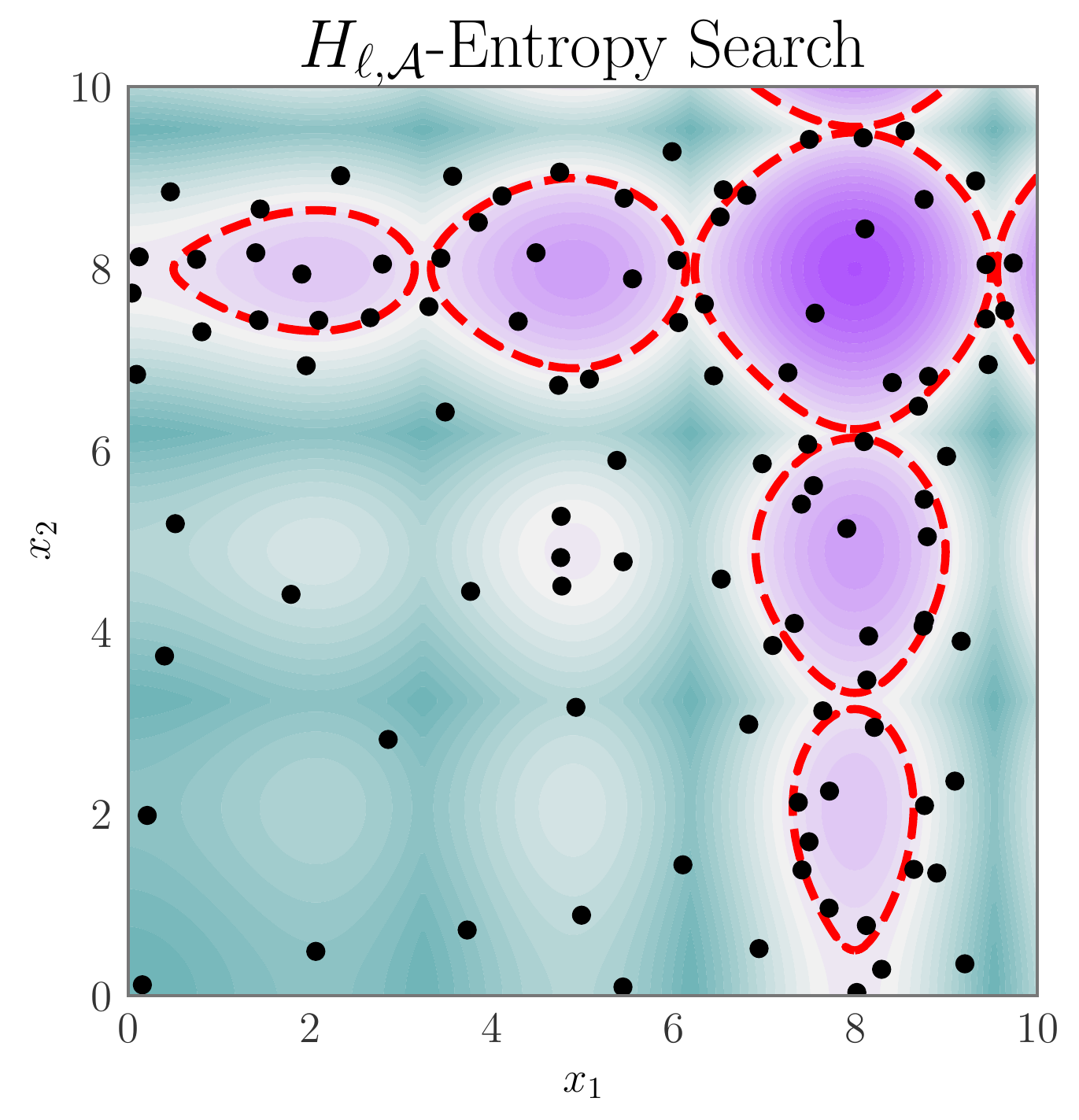}\\
\includegraphics[width=0.245\linewidth]{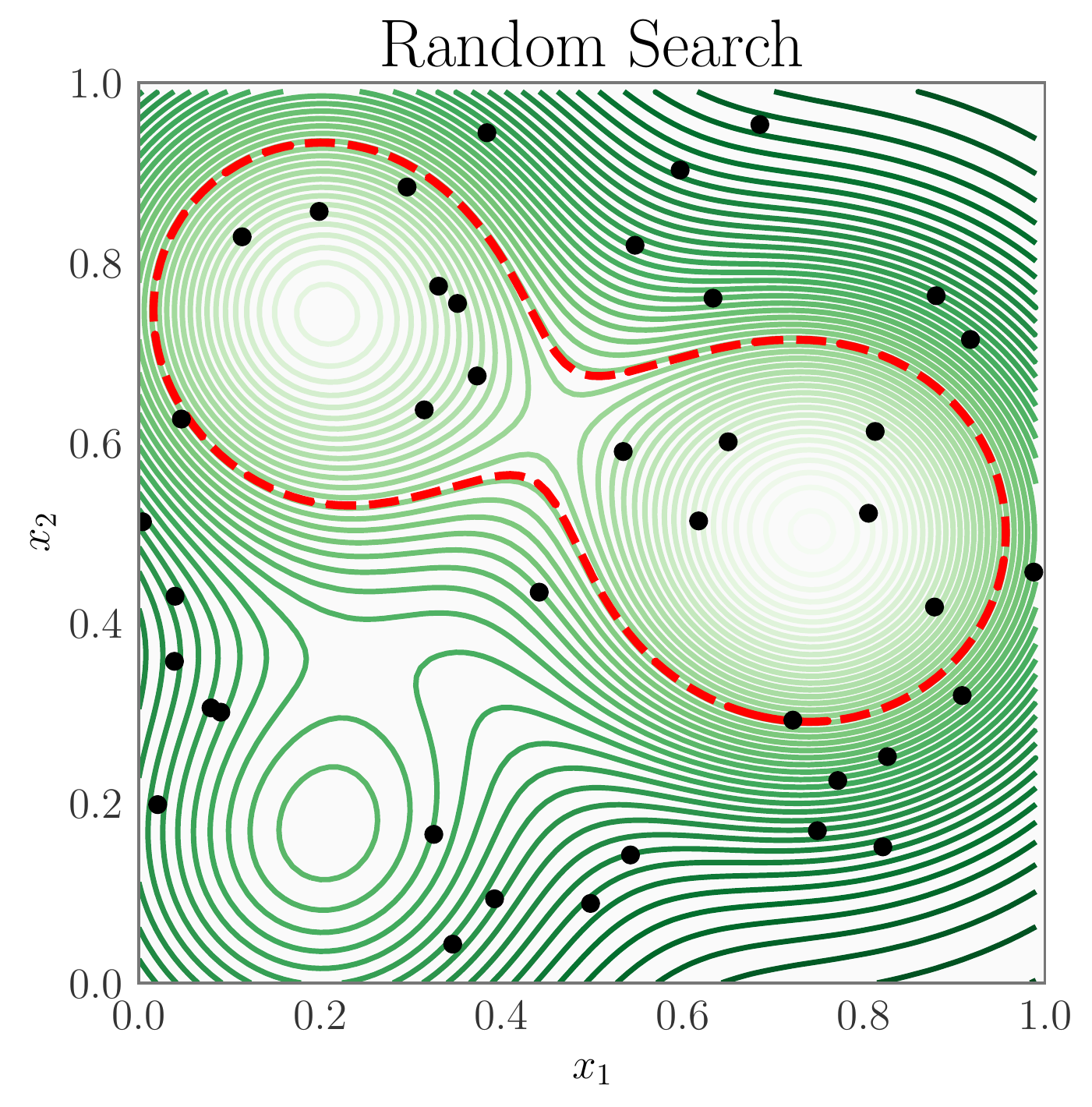}
\includegraphics[width=0.245\linewidth]{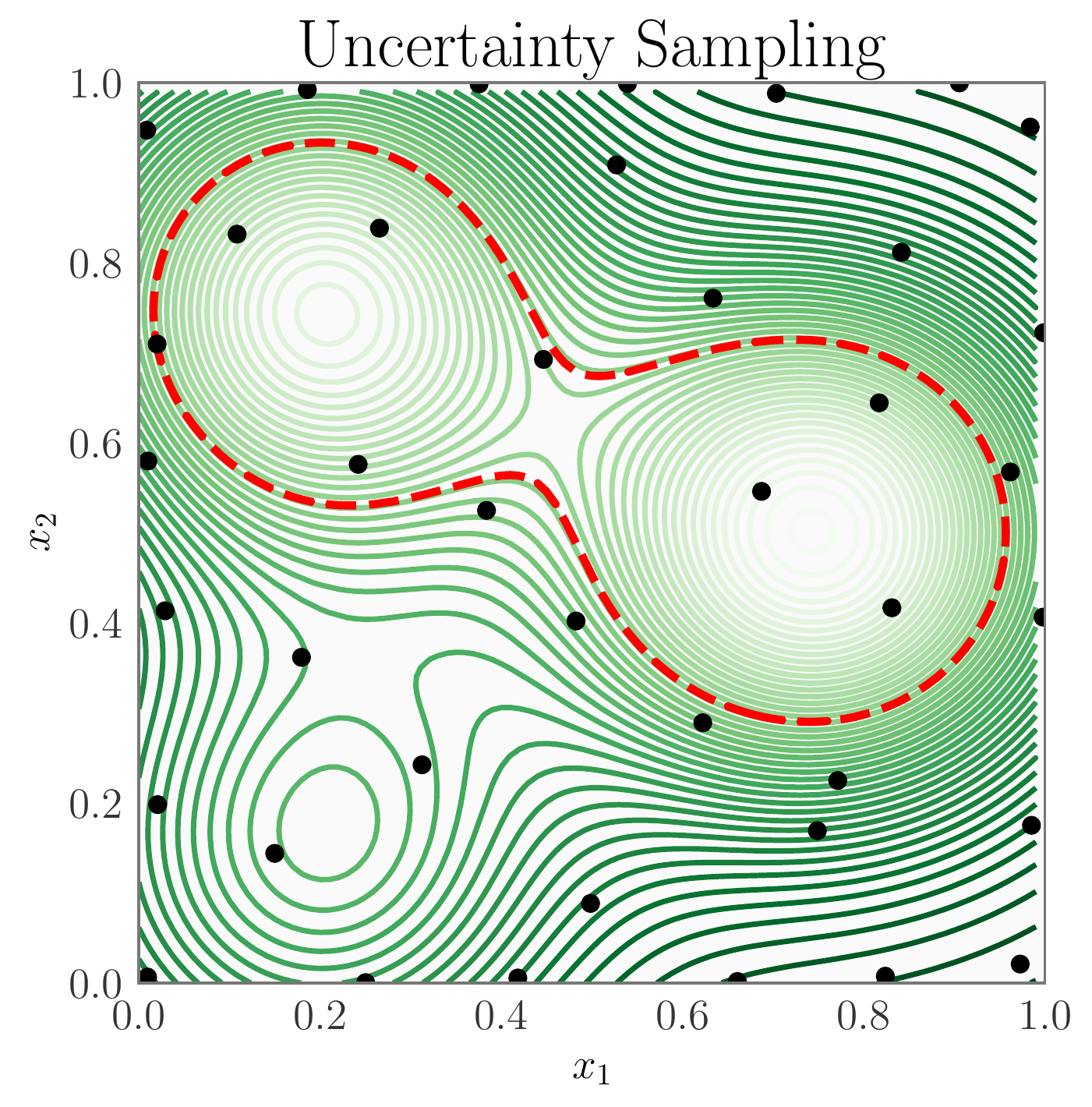}
\includegraphics[width=0.245\linewidth]{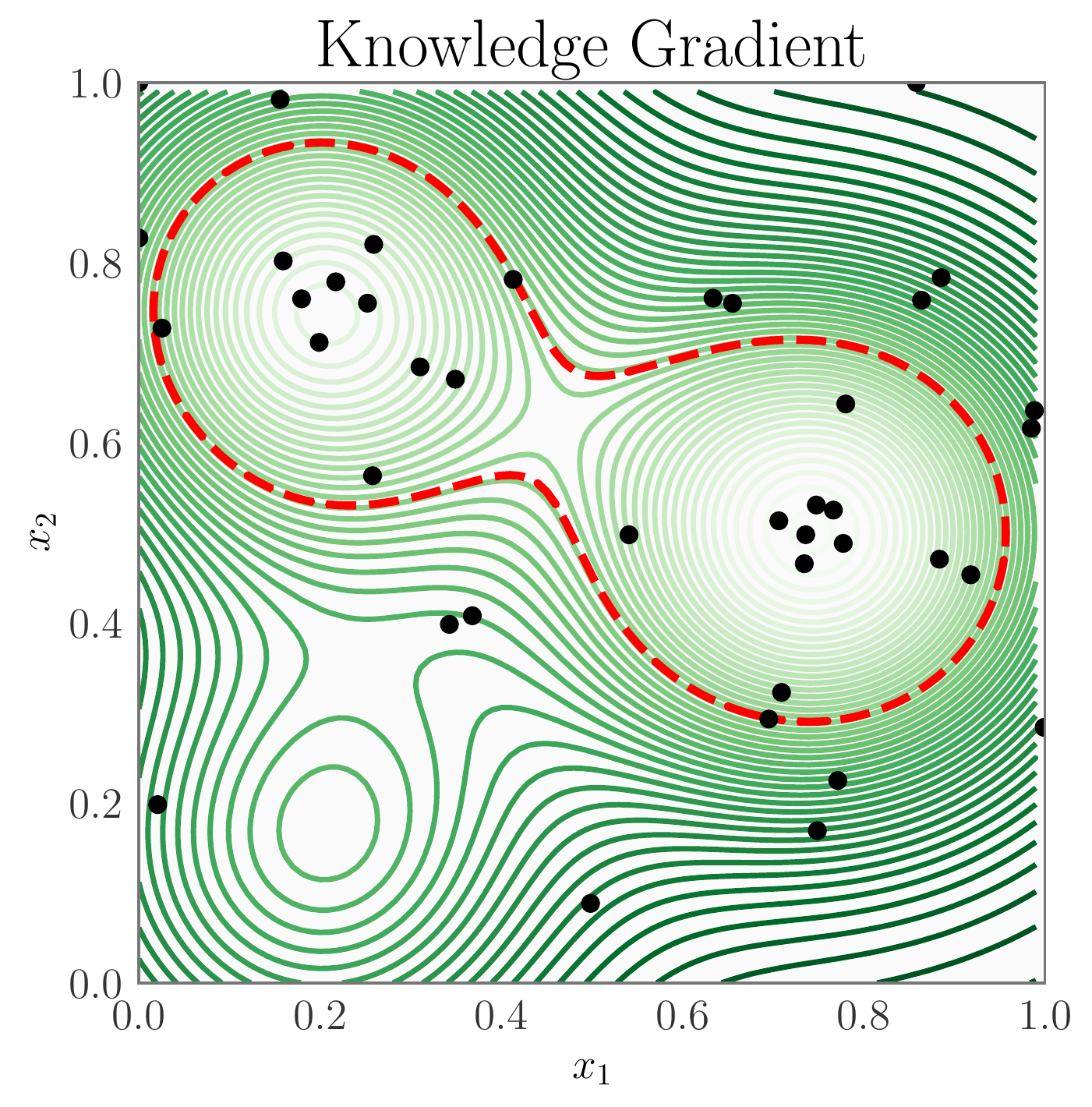}
\includegraphics[width=0.245\linewidth]{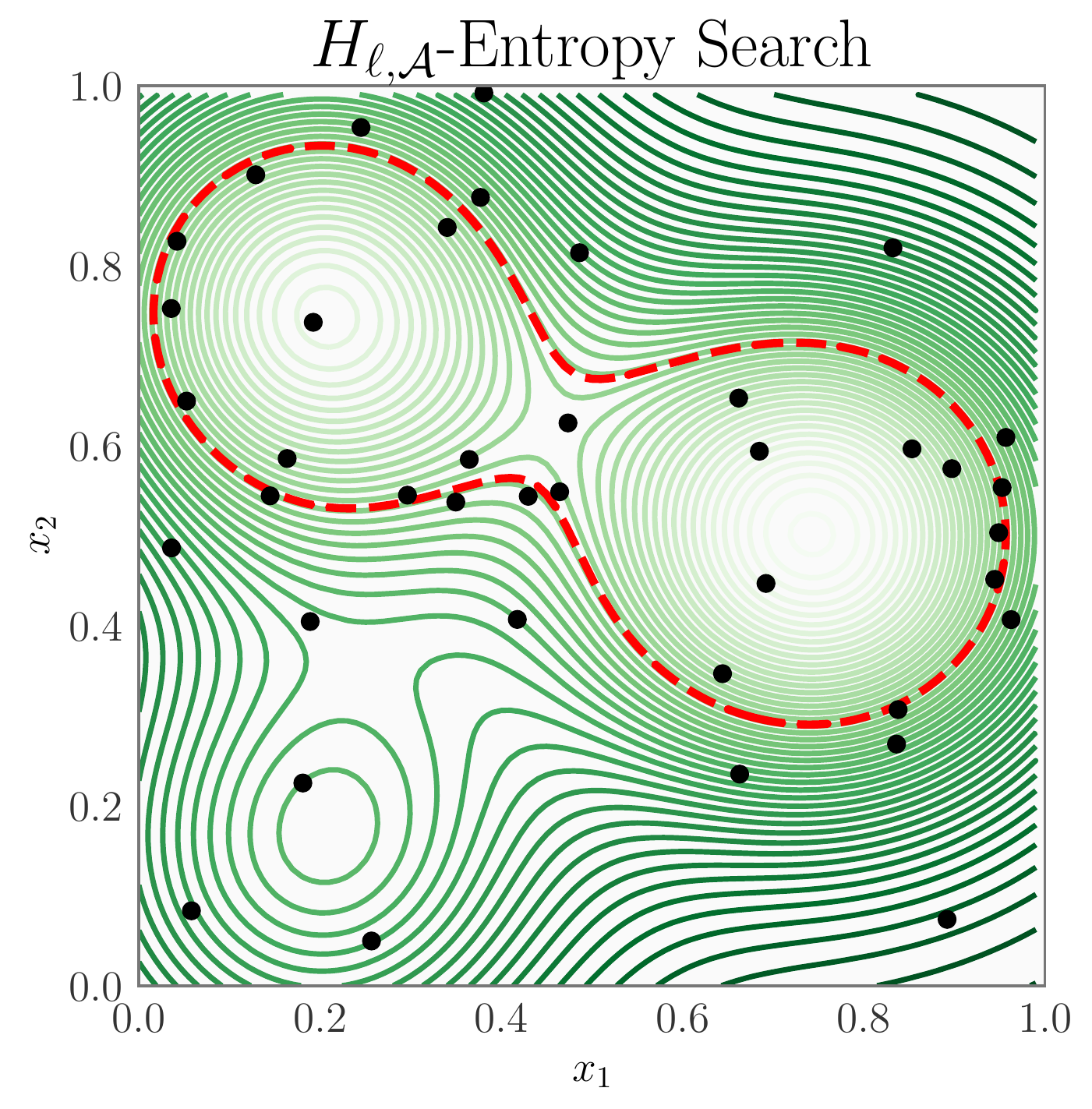}\\
\includegraphics[width=0.24\linewidth]{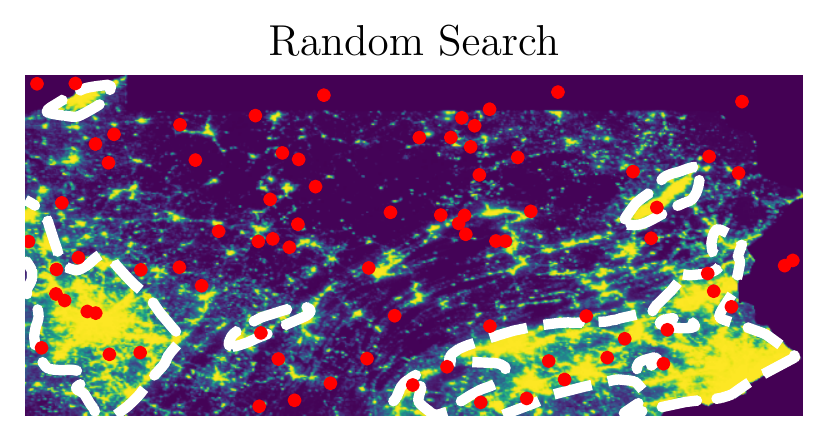}
\includegraphics[width=0.24\linewidth]{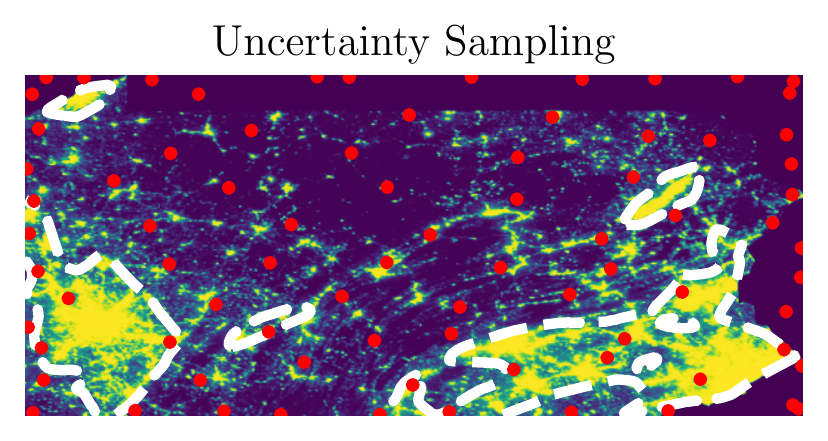}
\includegraphics[width=0.24\linewidth]{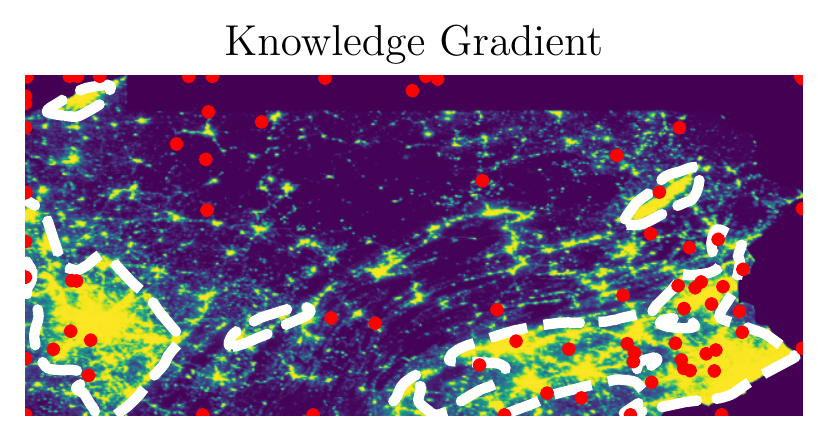}
\includegraphics[width=0.24\linewidth]{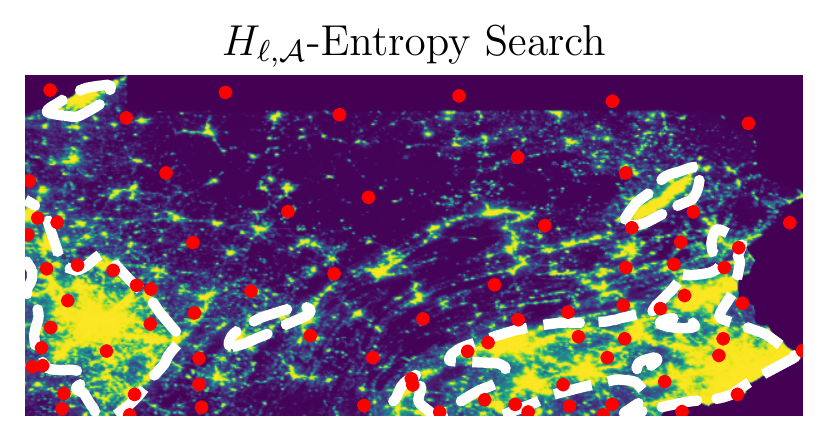}
\caption{\small
\textbf{Visualization results for multi-level set estimation.}
Visualization of multi-level set estimation for Alpine-$2$, Multihills, and
the Pennsylvania Night Light (PNL) functions. 
We show the ground-truth level set thresholds with red and blue dashed lines (for Alpine-$2$ and Multihills)
and white dashed line (for the PNL function). The queries $\Dc_t$ taken by each method are shown with black
dots (for Alpine-$2$ and Multihills) and red dots (for the PNL function). We observe that the queries taken
by $H_{\ell, \mathcal{A}}$-Entropy Search focus on level set boundaries, yielding a fine-grained estimate
near these boundary curves, while the other methods fail to do so.
}
\label{fig:app-levelset}
\end{figure}


\end{document}